\documentclass[10pt]{article}
\usepackage{algorithm}
\usepackage{algorithmic}

\usepackage{pdfpages, fullpage}

\usepackage{times}

\usepackage[utf8]{inputenc} 
\usepackage[T1]{fontenc}    
\usepackage{hyperref}       
\usepackage{url}            
\usepackage{booktabs}       
\usepackage{amsfonts}       
\usepackage{nicefrac}       
\usepackage{microtype}      
\usepackage{enumitem}

\usepackage{amsmath}
\usepackage{amssymb}
\usepackage{amsthm}
\usepackage{textcomp}
\usepackage{gensymb}
\usepackage{graphicx}
\usepackage{setspace} 
\usepackage{subfig}

\newtheorem{theorem}{Theorem}
\newtheorem{definition}[theorem]{Definition}
\newtheorem{lemma}[theorem]{Lemma}

\DeclareMathOperator*{\argmin}{\arg\!\min}
\newcommand{\matbeg}{\left(\begin{array}}
\newcommand{\matend}{\end{array}\right)}

\newcommand{\mathbbm}[1]{\text{\usefont{U}{bbm}{m}{n}#1}}
\usepackage{eurosym}
\DeclareRobustCommand{\officialeuro}{%
  \ifmmode\expandafter\text\fi
  {\fontencoding{U}\fontfamily{eurosym}\selectfont e}}

\newcommand{\sm}{\setminus}

\title{Learning Restricted Boltzmann Machines\\ with Sparse Latent Variables}

\author{
    Guy Bresler\footnote{Massachusetts Institute of Technology. Department of EECS. Email: \texttt{guy@mit.edu}.}
    \and 
    Rares-Darius Buhai\footnote{Massachusetts Institute of Technology. Department of EECS. Email: \texttt{rbuhai@mit.edu}. Current affiliation: ETH Zurich. Computer Science Department. Email: \texttt{rares.buhai@inf.ethz.ch}.}
}

\begin{document}

\maketitle

\begin{abstract}
Restricted Boltzmann Machines (RBMs) are a common family of undirected graphical models with latent variables. An RBM is described by a bipartite graph, with all observed variables in one layer and all latent variables in the other. We consider the task of learning an RBM given samples generated according to it. The best algorithms for this task currently have time complexity $\tilde{O}(n^2)$ for ferromagnetic RBMs (i.e., with attractive potentials) but $\tilde{O}(n^d)$ for general RBMs, where $n$ is the number of observed variables and $d$ is the maximum degree of a latent variable. Let the \textit{MRF neighborhood} of an observed variable be its neighborhood in the Markov Random Field of the marginal distribution of the observed variables. In this paper, we give an algorithm for learning general RBMs with time complexity $\tilde{O}(n^{2^s+1})$, where $s$ is the maximum number of latent variables connected to the MRF neighborhood of an observed variable. This is an improvement when $s < \log_2 (d-1)$, which corresponds to RBMs with sparse latent variables. Furthermore, we give a version of this learning algorithm that recovers a model with small prediction error and whose sample complexity is independent of the minimum potential in the Markov Random Field of the observed variables. This is of interest because the sample complexity of current algorithms scales with the inverse of the minimum potential, which cannot be controlled in terms of natural properties of the RBM.
\end{abstract}

\section{Introduction}

\subsection{Background}

Undirected graphical models, also known as \textit{Markov Random Fields} (MRFs), are probabilistic models in which a set of random variables is described with the help of an undirected graph, such that the graph structure corresponds to the dependence relations between the variables. Under mild conditions, the distribution of the random variables is determined by potentials associated with each clique of the graph \cite{hammersley1971markov}.

The joint distribution of any set of random variables can be represented as an MRF on a complete graph. However, MRFs become useful when the graph has nontrivial structure, such as bounded degree or bounded clique size. In such cases, learning and inference can often be carried out with greater efficiency. Since many phenomena of practical interest can be modelled as MRFs (e.g., magnetism \cite{brush1967history}, images \cite{li2012markov}, gene interactions and protein interactions \cite{wei2007markov, deng2002prediction}), it is of great interest to understand the complexity, both statistical and computational, of algorithmic tasks in these models.

The expressive power of graphical models is significantly strengthened by the presence of latent variables, i.e., variables that are not observed in samples generated according to the model. However, algorithmic tasks are typically more difficult in models with latent variables. Results on learning models with latent variables include \cite{mossel2005learning} for hidden Markov models, \cite{choi2011learning} for tree graphical models, \cite{chandrasekaran2010latent} for Gaussian graphical models, and \cite{anandkumar2013learning} for locally tree-like graphical models with correlation decay.

In this paper we focus on the task of learning \textit{Restricted Boltzmann Machines} (RBMs) \cite{smolensky1986information, freund1992unsupervised, hinton2002training}, which are a family of undirected graphical models with latent variables. The graph of an RBM is bipartite, with all observed variables in one layer and all latent variables in the other. This encodes the fact that the variables in one layer are jointly independent conditioned on the variables in the other layer. In practice, RBMs are used to model a set of observed features as being influenced by some unobserved and independent factors; this corresponds to the observed variables and the latent variables, respectively. RBMs are useful in common factor analysis tasks such as collaborative filtering \cite{salakhutdinov2007restricted} and topic modelling \cite{hinton2009replicated}, as well as in applications in domains as varied as speech recognition \cite{jaitly2011learning}, healthcare \cite{yan2015restricted}, and quantum mechanics \cite{nomura2017restricted}.

In formalizing the learning problem, a challenge is that there are infinitely many RBMs that induce the same marginal distribution of the observed variables. To sidestep this non-identifiability issue, the literature on learning RBMs focuses on learning the marginal distribution itself. This marginal distribution is, clearly, an MRF. Call the \textit{order} of an MRF the size of the largest clique that has a potential. Then, more specifically, it is known that the marginal distribution of the observed variables is an MRF of order at most $d$, where $d$ is the maximum degree of a latent variable in the RBM. Hence, one way to learn an RBM is to simply apply algorithms for learning MRFs.
The best current algorithms for learning MRFs have time complexity $\tilde{O}(n^r)$, where $r$ is the order of the MRF \cite{hamilton2017information,klivans2017learning,vuffray2019efficient}. Applying these algorithms to learning RBMs therefore results in time complexity $\tilde{O}(n^d)$. We note that these time complexities hide the factors that do not depend on $n$.

This paper is motivated by the following basic question: 
\[
\text{\emph{In what settings is it possible to learn RBMs with time complexity substantially better than $\tilde{O}(n^d)$?}}\] 
Reducing the runtime of learning arbitrary MRFs of order $r$ to below $n^{\Omega(r)}$ is unlikely, because learning such MRFs subsumes learning noisy parity over $r$ bits \cite{bresler2014structure}, and it is widely believed that learning $r$-parities with noise (LPN) requires time $n^{\Omega(r)}$ \cite{kearns1998efficient}. 
For ferromagnetic RBMs, i.e., RBMs with non-negative interactions, \cite{bresler2019learning} gave an algorithm with time complexity $\tilde{O}(n^2)$. In the converse direction, \cite{bresler2019learning} gave a general reduction from learning MRFs of order $r$ to learning (non-ferromagnetic) RBMs with maximum degree of a latent variable $r$.
 
In other words, the problem of learning RBMs is just as challenging as for MRFs, and therefore learning general RBMs cannot be done in time less than $n^{\Omega(d)}$ without violating conjectures about LPN. 

The reduction in \cite{bresler2019learning} from learning order $r$ MRFs to learning RBMs uses an \emph{exponential} in $r$ number of latent variables to represent each neighborhood of the MRF. Thus, there is hope that RBMs with \emph{sparse} latent variables are in fact easier to learn than general MRFs. The results of this paper demonstrate that this is indeed the case. 

\subsection{Contributions}

Let the \textit{MRF neighborhood} of an observed variable be its neighborhood in the MRF of the marginal distribution of the observed variables. Let $s$ be the maximum number of latent variables connected to the MRF neighborhood of an observed variable. We give an algorithm with time complexity $\tilde{O}(n^{2^s+1})$ that recovers with high probability the MRF neighborhoods of all observed variables. This represents an improvement over current algorithms when $s < \log_2(d-1)$.

The reduction in time complexity is made possible by the following key structural result: if the mutual information $I(X_u; X_I | X_S)$ is large for some observed variable $X_u$ and some subsets of observed variables $X_I$ and $X_S$, then there exists a subset $I'$ of $I$ with $|I'| \leq 2^s$ such that $I(X_u ; X_{I'}|X_S)$ is also large. This result holds because of the special structure of the RBM, in which, with few latent variables connected to the neighborhood of any observed variable, not too many of the low-order potentials of the induced MRF can be cancelled. 

Our algorithm is an extension of the algorithm of \cite{hamilton2017information} for learning MRFs. To find the neighborhood of a variable $X_u$, their algorithm iteratively searches over all subsets of variables $X_I$ with $|I| \leq d-1$ for one with large mutual information $I(X_u; X_I|X_S)$, which is then added to the current set of neighbors $X_S$. Our structural result implies that it is sufficient to search over subsets $X_I$ with $|I| \leq 2^s$, which reduces the time complexity from $\tilde{O}(n^{d})$ to $\tilde{O}(n^{2^s+1})$.

For our algorithm to be advantageous, it is necessary that $s < \log_2(d-1)$. Note that $s$ is implicitly also an upper bound on the maximum degree of an observed variable in the RBM. Figure \ref{figure:rbm_d20_s3} shows an example of a class of RBMs for which our assumptions are satisfied.
In this example, $s$ can be made arbitrarily smaller than $d$, $n$, and the number of latent variables. 


\begin{figure}
    \centering
    \includegraphics[scale=0.26]{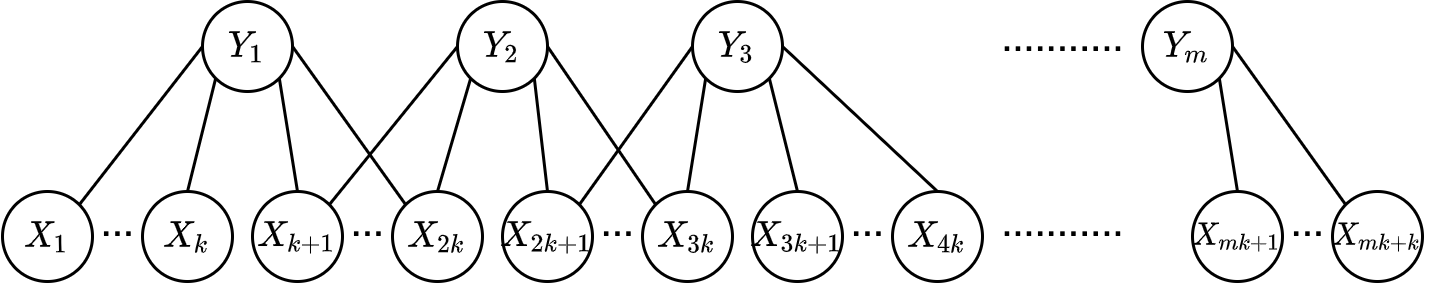}
\caption{Class of RBMs with $mk+k$ observed variables, $m$ latent variables, $d=2k$, and $s=4$. The $X$ variables represent observed variables, the $Y$ variables represent latent variables, and the edges represent non-zero interactions between variables. The ``$\cdots$'' hides variables that have consecutive indices. The variables hidden by ``$\cdots$'' have the same connections as the variables at the extremes of their respective dots.}
\label{figure:rbm_d20_s3}
\end{figure}

The sample complexity of our algorithm is the same as that of \cite{hamilton2017information}, with some additional factors due to working with subsets of size at most $2^s$. We extended \cite{hamilton2017information} instead of one of \cite{klivans2017learning,vuffray2019efficient}, which have better sample complexities, because our main goal was to improve the time complexity, and we found \cite{hamilton2017information} the most amenable to extensions in this direction. The sample complexity necessarily depends on the width (defined in Section \ref{sec:prelim}) and the minimum absolute-value non-zero potential of the MRF of the observed variables \cite{santhanam2012information}. In the Appendix F, we show that our sample complexity actually depends on a slightly weaker notion of MRF width than that used in current papers. This modified MRF width has a more natural correspondence with properties of the RBM.

The algorithm we described only recovers the structure of the MRF of the observed variables, and not its potentials. However, recovering the potentials is easy after the structure is known: e.g., see Section 6.2 in \cite{bresler2019learning}.

The second contribution of this paper is an algorithm for learning RBMs with time complexity $\tilde{O}(n^{2^s+1})$ whose sample complexity does not depend on the minimum potential of the MRF of the observed variables. The algorithm is not guaranteed to recover the correct MRF neighborhoods, but is guaranteed to recover a model with small prediction error (a distinction analogous to that between support recovery and prediction error in regression). This result is of interest because all current algorithms depend on the minimum potential, which can be degenerate even when the RBM itself has non-degenerate interactions. Learning graphical models in order to make predictions was considered before in \cite{bresler2016learning} for trees.

In more detail, we first give a structure learning algorithm that recovers the MRF neighborhoods corresponding to large potentials. Second, we give a regression algorithm that estimates the potentials corresponding to these MRF neighborhoods. Lastly, we quantify the error of the resulting model for predicting the value of an observed variable given the other observed variables. Overall, we achieve prediction error $\epsilon$ with a sample complexity that scales exponentially with $\epsilon^{-1}$, and that otherwise has dependencies comparable to our main algorithm.

\subsection{Overview of structural result}

We present now the intuition and techniques behind our structural result. Theorem \ref{thm:informal_nu_bound_alone} states an informal version of this result.

\begin{theorem}[Informal version of Theorem \ref{thm:nu_bound_alone}]
Fix observed variable $u$ and subsets of observed variables $I$ and $S$, such that all three are disjoint. Suppose that $I$ is a subset of the MRF neighborhood of $u$ and that $|I| \leq d - 1$. Then there exists a subset $I' \subseteq I$ with $|I'| \leq 2^s$ such that
\label{thm:informal_nu_bound_alone}
\[\nu_{u,I'|S} \geq C_{s,d} \cdot \nu_{u,I|S}\]
where $C_{s,d} > 0$ depends on $s$ and $d$, and where $\nu_{u,I',S}$ and $\nu_{u,I|S}$ are proxies of $I(X_u,X_{I'}|X_S)$ and $I(X_u,X_I|X_S)$, respectively.
\end{theorem}

The formal definition of $\nu$ is in Section \ref{sec:prelim}. For the purposes of this section, one can think of it as interchangeable with the mutual information. Furthermore, this section only discusses how to obtain a point-wise version of the bound, $\nu_{u,I'|S}(x_u,x_{I'}|x_S) \geq C'_{s,d} \cdot \nu_{u,I|S}(x_u,x_I|x_S)$, evaluated at specific $x_u$, $x_I$, and $x_S$. It is not too difficult to extend this result to $\nu_{u,I'|S} \geq C_{s,d} \cdot \nu_{u,I|S}$.

In general, estimating the MRF neighborhood of an observed variable is hard because the low-order information between the observed variables can vanish. In that case, to obtain any information about the distribution, it is necessary to work with high-order interactions of the observed variables. Typically, this translates into large running times. 

Theorem \ref{thm:informal_nu_bound_alone} shows that if there is some high-order $\nu_{u,I|S}$ that is non-vanishing, then there is also some $\nu_{u,I'|S}$ with $|I'| \leq 2^s$ that is non-vanishing. That is, the order up to which all the information can vanish is less than $2^s$. Or, in other words, RBMs in which all information up to a large order vanishes are complex and require \textit{many} latent variables.

To prove this result, we need to relate the mutual information in the MRF neighborhood of an observed variable to the number of latent variables connected to it. This is challenging because the latent variables have a non-linear effect on the distribution of the observed variables. This non-linearity makes it difficult to characterize what is ``lost'' when the number of latent variables is small.

The first main step of our proof is Lemma \ref{lemma:expr_nu}, which expresses $\nu_{u,I|S}(x_u,x_I|x_S)$ as a sum over $2^s$ terms, representing the configurations of the latent variables connected to $I$. Each term of the sum is a product over the observed variables in $I$. This expression is convenient because it makes explicit the contribution of the latent variables to $\nu_{u,I|S}(x_u,x_I|x_S)$. The proof of the lemma is an ``interchange of sums'', going from sums over configurations of observed variables to sums over configurations of latent variables.

The second main step is Lemma \ref{lemma_non_cancel_abstract}, which shows that for a sum over $m$ terms of products over $n$ terms, it is possible to reduce the number of terms in the products to $m$, while decreasing the original expression by at most a factor of $C'_{m,n}$, for some $C'_{m,n} > 0$ depending on $n$ and $m$. Combined with Lemma \ref{lemma:expr_nu}, this result implies the existence of a subset $I'$ with $|I'| \leq 2^s$ such that $\nu_{u,I'|S}(x_u,x_{I'}|x_S) \geq C'_{s,d} \cdot \nu_{u,I|S}(x_u,x_I|x_S)$.

\section{Preliminaries and notation}
\label{sec:prelim}

We start with some general notation: $[n]$ is the set $\{1, ..., n\}$; $\mathbbm{1}\{A\}$ is $1$ if the statement $A$ is true and $0$ otherwise; $\binom{n}{k}$ is the binomial coefficient $\frac{n!}{k!(n-k)!}$; $\sigma(x)$ is the sigmoid function $\sigma(x) = \frac{1}{1+e^{-x}}$.

\begin{definition}
A \textnormal{Markov Random Field}\footnote{This definition holds if each assignment of the random variables has positive probability, which is satisfied by the models considered in this paper.} of order $r$ is a distribution over random variables $X \in \{-1, 1\}^n$ with probability mass function
\[\mathbb{P}(X=x) \propto \exp(f(x))\]
where $f$ is a polynomial of order $r$ in the entries of $x$.
\end{definition}
Because $x \in \{-1,1\}^n$, it follows that $f$ is a multilinear polynomial, so it can be represented as 
\[f(x) = \sum_{S \subseteq [n]} \hat{f}(S) \chi_S(x)\]
where $\chi_S(x) = \prod_{i \in S} x_i$. The term $\hat{f}(S)$ is called the Fourier coefficient corresponding to $S$, and it represents the potential associated with the clique $\{X_i\}_{i \in S}$ in the MRF. There is an edge between $X_i$ and $X_j$ in the MRF if and only if there exists some $S \subseteq[n]$ such that $i, j \in S$ and $\hat{f}(S) \neq 0$. Some other relevant notation for MRFs is: let $D$ be the maximum degree of a variable; let $\alpha$ be the minimum absolute-value non-zero Fourier coefficient; let $\gamma$ be the width:
\[\gamma := \max_{u \in [n]} \sum_{\substack{S \subseteq [n]\\ u \in S}} |\hat{f}(S)|.\]

\begin{definition}
A \textnormal{Restricted Boltzmann Machine} is a distribution over observed random variables $X \in \{-1, 1\}^n$ and latent random variables $Y \in \{-1, 1\}^m$ with probability mass function
\[\mathbb{P}(X=x,Y=y) \propto \exp\left( x^T J y + h^T x + g^T y \right)\]
where $J \in \mathbb{R}^{n \times m}$ is an interaction (or weight) matrix, $h \in \mathbb{R}^n$ is an external field (or bias) on the observed variables, and $g \in \mathbb{R}^m$ is an external field (or bias) on the latent variables.
\end{definition}

There exists an edge between $X_i$ and $Y_j$ in the RBM if and only if $J_{i,j} \neq 0$. The resulting graph is bipartite, and all the variables in one layer are conditionally jointly independent given the variables in the other layer. Some other relevant notation for RBMs is: let $d$ be the maximum degree of a latent variable; let $\alpha^*$ be the minimum absolute-value non-zero interaction; let $\beta^*$ be the width:
\[\beta^* := \max\left( \max_{i \in [n]} \sum_{j=1}^m |J_{i,j}| + |h_i|, \max_{j \in [m]} \sum_{i=1}^n |J_{i,j}| + |g_j| \right).\]
In the notation above, we say that an RBM is $(\alpha^*, \beta^*)$-consistent. Typically, to ensure that the RBM is non-degenerate, it is required for $\alpha^*$ not to be too small and for $\beta^*$ not to be too large; otherwise, interactions can become undetectable or deterministic, respectively, both of which lead to non-identifiability \cite{santhanam2012information}.

In an RBM, it is known that there is a lower bound of $\sigma(-2\beta^*)$ and an upper bound of $\sigma(2\beta^*)$ on any probability of the form
\[\mathbb{P}(X_u=x_u|E) \quad \text{ or } \quad \mathbb{P}(Y_u=y_u|E)\]
where $E$ is any event that involves the other variables in the RBM. It is also known that the marginal distribution of the observed variables is given by (e.g., see Lemma 4.3 in \cite{bresler2019learning}):
\[\mathbb{P}(X=x) \propto \exp(f(x)) = \exp\left( \sum_{j=1}^m \rho(J_j \cdot x + g_j) + h^T x \right)\]
where $J_j$ is the $j$-th column of $J$ and $\rho(x) = \log(e^x+e^{-x})$. From this, it can be shown that the marginal distribution is an MRF of order at most $d$. 

We now define $s$, the maximum number of latent variables connected to the MRF neighborhood of an observed variable:
\[s := \max_{u \in [n]} \sum_{j=1}^m \mathbbm{1}\{\exists i \in [n] \setminus \{u\} \text{ and } S \subseteq [n] \text{ s.t. } u, i \in S \text{ and } \hat{f}(S) \neq 0 \text{ and } J_{i,j} \neq 0\}.\]
The MRF neighborhood of an observed variable is a subset of the two-hop neighborhood of the observed variable in the RBM; typically the two neighborhoods are identical. Therefore, an upper bound on $s$ is obtained as the maximum number of latent variables connected to the two-hop neighborhood of an observed variable in the RBM.

Finally, we define a proxy to the conditional mutual information, which is used extensively in our analysis. For random variables $X_u \in \{-1,1\}$, $X_I \in \{-1,1\}^{|I|}$, and $X_S \in \{-1,1\}^{|S|}$, let
\[\nu_{u,I|S} := \mathbb{E}_{R,G}\left[ \mathbb{E}_{X_S} \left[\left| \mathbb{P}(X_u=R,X_I=G|X_S) - \mathbb{P}(X_u=R|X_S) \mathbb{P}(X_I=G|X_S) \right|\right] \right]\]
where $R$ and $G$ come from uniform distributions over $\{-1,1\}$ and $\{-1,1\}^{|I|}$, respectively. This quantity forms a lower bound on the conditional mutual information (e.g., see Lemma 2.5 in \cite{hamilton2017information}):
\[\sqrt{\frac{1}{2} I(X_u;X_I|X_S)} \geq \nu_{u,I|S}.\]
We also define an empirical version of this proxy, with the probabilities and the expectation over $X_S$ replaced by their averages from samples:
\[\hat{\nu}_{u,I|S} := \mathbb{E}_{R,G}\left[ \hat{\mathbb{E}}_{X_S} \left[\left| \hat{\mathbb{P}}(X_u=R,X_I=G|X_S) - \hat{\mathbb{P}}(X_u=R|X_S) \hat{\mathbb{P}}(X_I=G|X_S) \right|\right] \right].\]

\section{Learning Restricted Boltzmann Machines with sparse latent variables}
\label{sec:algorithm}

To find the MRF neighborhood of an observed variable $u$ (i.e., observed variable $X_u$; we use the index and the variable interchangeably when no confusion is possible), our algorithm takes the following steps, similar to those of the algorithm of \cite{hamilton2017information}:
\begin{enumerate}
    \item Fix parameters $s$, $\tau'$, $L$. Fix observed variable $u$. Set $S:=\emptyset$.
    \item While $|S|\leq L$ and there exists a set of observed variables $I \subseteq [n] \setminus \{u\} \setminus S$ of size at most $2^{s}$ such that $\hat{\nu}_{u,I|S} > \tau'$, set $S:=S\cup I$.
    \item For each $i \in S$, if $\hat{\nu}_{u,i|S\setminus \{i\}} < \tau'$, remove $i$ from $S$.
    \item Return set $S$ as an estimate of the neighborhood of $u$.
\end{enumerate}
We use
\[L = 8/(\tau')^2, \quad \tau' = \frac{1}{(4d)^{2^{s}}} \left(\frac{1}{d}\right)^{2^{s}(2^{s}+1)} \tau, \text{ and } \tau = \frac{1}{2} \frac{4\alpha^2 (e^{-2\gamma})^{d+D-1} }{d^{4d} 2^{d+1} \binom{D}{d-1} \gamma e^{2\gamma}},\]
where $\tau$ is exactly as in \cite{hamilton2017information} when adapted to the RBM setting. In the above, $d$ is a property of the RBM, and $D$, $\alpha$, and $\gamma$ are properties of the MRF of the observed variables. 

With high probability, Step 2 is guaranteed to add to $S$ all the MRF neighbors of $u$, and Step 3 is guaranteed to prune from $S$ any non-neighbors of $u$. Therefore, with high probability, in Step 4 $S$ is exactly the MRF neighborhood of $u$. In the original algorithm of \cite{hamilton2017information}, the guarantees of Step 2 were based on this result: if $S$ does not contain the entire neighborhood of $u$, then $\nu_{u,I|S} \geq 2\tau$ for some set $I$ of size at most $d-1$. As a consequence, Step 2 entailed a search over size $d-1$ sets. The analogous result in our setting is given in Theorem \ref{thm:main_nu_bound}, which guarantees the existence of a set $I$ of size at most $2^s$, thus reducing the search to sets of this size. This theorem follows immediately from Theorem \ref{thm:nu_bound_alone}, the key structural result of our paper.

\begin{theorem}
\label{thm:nu_bound_alone}
Fix observed variable $u$ and subsets of observed variables $I$ and $S$, such that all three are disjoint. Suppose that $I$ is a subset of the MRF neighborhood of $u$ and that $|I| \leq d-1$. Then there exists a subset $I' \subseteq I$ with $|I'| \leq 2^s$ such that
\[\nu_{u,I'|S} \geq \frac{1}{(4d)^{2^{s}}} \left(\frac{1}{d}\right)^{2^{s}(2^{s}+1)} \nu_{u, I|S}.\]
\end{theorem}


Using the result in Theorem \ref{thm:nu_bound_alone}, we now state and prove Theorem \ref{thm:main_nu_bound}.

\begin{theorem}
\label{thm:main_nu_bound}
Fix an observed variable $u$ and a subset of observed variables $S$, such that the two are disjoint. Suppose that $S$ does not contain the entire MRF neighborhood of $u$. Then there exists some subset $I$ of the MRF neighborhood of $u$ with $|I| \leq 2^s$ such that
\[\nu_{u,I|S} \geq \frac{1}{(4d)^{2^{s}}} \left(\frac{1}{d}\right)^{2^{s}(2^{s}+1)} \frac{4\alpha^2 (e^{-2\gamma})^{d+D-1} }{d^{4d} 2^{d+1} \binom{D}{d-1} \gamma e^{2\gamma}} = 2 \tau'.\]
\end{theorem}
\begin{proof}
By Theorem 4.6 in \cite{hamilton2017information}, we have that there exists some subset $I$ of neighbors of $u$ with $|I| \leq d-1$ such that 
\[\nu_{u,I|S} \geq \frac{4\alpha^2 (e^{-2\gamma})^{d+D-1} }{d^{4d} 2^{d+1} \binom{D}{d-1} \gamma e^{2\gamma}} = 2 \tau.\]
Then, by Theorem \ref{thm:nu_bound_alone}, we have that there exists some subset $I' \subseteq I$ with $|I'| \leq 2^s$ such that
\[\nu_{u,I'|S} \geq \frac{1}{(4d)^{2^{s}}} \left(\frac{1}{d}\right)^{2^{s}(2^{s}+1)} 2 \tau = \frac{1}{(4d)^{2^{s}}} \left(\frac{1}{d}\right)^{2^{s}(2^{s}+1)} \frac{4\alpha^2 (e^{-2\gamma})^{d+D-1} }{d^{4d} 2^{d+1} \binom{D}{d-1} \gamma e^{2\gamma}} = 2 \tau'.\]
\end{proof}

Theorem \ref{thm:algorithm} states the guarantees of our algorithm. The analysis is very similar to that in \cite{hamilton2017information}, and is deferred to the Appendix B. Then, Section \ref{sec:structural} sketches the proof of Theorem \ref{thm:nu_bound_alone}.

\begin{theorem}
\label{thm:algorithm}
Fix $\omega > 0$. Suppose we are given $M$ samples from an RBM, where
\[M \geq \frac{60 \cdot 2^{2L}}{(\tau')^2 (e^{-2\gamma})^{2L}} \left( \log(1/\omega) + \log(L+2^s+1) + (L+2^s+1)\log(2n) + \log 2 \right).\]
Then with probability at least $1-\omega$, our algorithm, when run from each observed variable $u$, recovers the correct neighborhood of $u$. Each run of the algorithm takes $O(MLn^{2^s+1})$ time.
\end{theorem}

\section{Proof sketch of structural result}
\label{sec:structural}
The proofs of the lemmas in this section can be found in the Appendix A.
Consider the mutual information proxy when the values of $X_u$, $X_I$, and $X_S$ are fixed:
\begin{align*}
&\nu_{u,I|S}(x_u,x_I|x_S)\\
&\quad = \left| \mathbb{P}(X_u=x_u,X_I=x_I|X_S=x_S) - \mathbb{P}(X_u=x_u|X_S=x_S) \mathbb{P}(X_I=x_I|X_S=x_S) \right|.
\end{align*}

We first establish a version of Theorem \ref{thm:nu_bound_alone} for $\nu_{u,I|S}(x_u,x_I|x_S)$, and then generalize it to $\nu_{u,I|S}$.

In Lemma \ref{lemma:expr_nu}, we express $\nu_{u,I|S}(x_u,x_I|x_S)$ as a sum over configurations of latent variables $U$ connected to observed variables in $I$. Note that $|U| \leq s$, so the summation is over at most $2^s$ terms. 

\begin{lemma}
\label{lemma:expr_nu}
Fix observed variable $u$ and subsets of observed variables $I$ and $S$, such that all three are disjoint. Suppose that $I$ is a subset of the MRF neighborhood of $u$. Then
\[\nu_{u,I|S}(x_{u},x_{I}|x_S)= \left| \sum_{q_U \in \{-1,1\}^{|U|}} \left(\sum_{q_{\sim U} \in \{-1,1\}^{m - |U|}} \bar{f}(q, x_u, x_S)\right) \prod_{i \in I} \sigma(2x_i(J^{(i)}\cdot q + h_i)) \right|\]
for some function $\bar{f}$, where $U$ is the set of latent variables connected to observed variables in $I$, $J^{(i)}$ is the $i$-th row of $J$, and the entries of $q_{\sim U}$ in the expression $J^{(i)} \cdot q$ are arbitrary.
\end{lemma}

Lemma \ref{lemma_non_cancel_abstract} gives a generic non-cancellation result for expressions of the form $\left|\sum_{i=1}^m a_i \prod_{j=1}^n x_{i,j}\right|$. 
Then, Lemma \ref{lemma_non_cancel_nu_indiv_cond} applies this result to the form of $\nu_{u,I|S}(x_u,x_I|x_S)$ in Lemma \ref{lemma:expr_nu}, and guarantees the existence of a subset $I' \subseteq I$ with $|I'| \leq 2^s$ such that $\nu_{u,I'|S}(x_u,x_{I'}|x_S)$ is within a bounded factor of $\nu_{u,I|S}(x_u,x_I|x_S)$.

\begin{lemma}
\label{lemma_non_cancel_abstract}
Let $x_{1,1}, ..., x_{m,n} \in [-1, 1]$, with $n > m$. Then, for any $a \in \mathbb{R}^m$, there exists a subset $S \subseteq [n]$ with $|S| \leq m$ such that
\[\left| \sum_{i=1}^m a_i \prod_{j \in S} x_{i,j} \right| \geq \frac{1}{4^m} \left(\frac{1}{n}\right)^{m(m+1)} \left| \sum_{i=1}^m a_i \prod_{j=1}^{n} x_{i,j} \right|.\]
\end{lemma}

We remark that, in this general form, Lemma \ref{lemma_non_cancel_abstract} is optimal in the size of the subset that it guarantees not to be cancelled. That is, there are examples with $\sum_{i=1}^m a_i \prod_{j=1}^n x_{i,j} \neq 0$ but $\sum_{i=1}^m a_i \prod_{j \in S} x_{i,j} = 0$ for all subsets $S \subseteq [n]$ with $|S| \leq m - 1$. See the Appendix A for a more detailed discussion. 

\begin{lemma}
\label{lemma_non_cancel_nu_indiv_cond}
Fix observed variable $u$ and subsets of observed variables $I$ and $S$, such that all three are disjoint. Suppose that $I$ is a subset of the MRF neighborhood of $u$. Fix any assignments $x_u$, $x_I$, and $x_S$. Then there exists a subset $I' \subseteq I$ with $|I'| \leq 2^s$ such that
\[\nu_{u, I'|S}(x_u, x_{I'}|x_S) \geq \frac{1}{4^{2^s}} \left(\frac{1}{|I|}\right)^{2^s(2^s+1)} \nu_{u, I | S}(x_u, x_{I} | x_S)\]
where $x_{I'}$ agrees with $x_I$.
\end{lemma}

Finally, Lemma \ref{lemma_non_cancel_nu_cond_full} extends the result about $\nu_{u, I | S}(x_u, x_{I} | x_S)$ to a result about $\nu_{u,I|S}$. The difficulty lies in the fact that the subset $I'$ guaranateed to exist in Lemma \ref{lemma_non_cancel_nu_indiv_cond} may be different for different configurations $(x_u, x_I, x_S)$. Nevertheless, the number of subsets $I'$ with $|I'| \leq 2^s$ is smaller than the number of configurations $(x_u, x_I, x_S)$, so we obtain a viable bound via the pigeonhole principle.

\begin{lemma}
\label{lemma_non_cancel_nu_cond_full}
Fix observed variable $u$ and subsets of observed variables $I$ and $S$, such that all three are disjoint. Suppose that $I$ is a subset of the MRF neighborhood of $u$. Then there exists a subset $I' \subseteq I$ with $|I'| \leq 2^s$ such that
\[\nu_{u, I' | S} \geq\frac{1}{(4|I|)^{2^s}} \left( \frac{1}{|I|} \right)^{2^s(2^s+1)} \nu_{u, I | S}.\]
\end{lemma}

This result completes the proof of Theorem \ref{thm:nu_bound_alone}. 

\section{Making good predictions independently of the minimum potential}
\label{sec:alpha}

Figure \ref{fig:cov_go_zero} shows an RBM for which $\alpha$ can be arbitrarily small, while $\alpha^* = 1$ and $\beta^*=2$. That is, the induced MRF can be degenerate, while the RBM itself has interactions that are far from degenerate. This is problematic: the sample complexity of our algorithm, which scales with the inverse of $\alpha$, can be arbitrarily large, even for seemingly well-behaved RBMs. 
In particular, we note that $\alpha$ is an opaque property of the RBM, and it is \emph{a priori} unclear how small it is. 

We emphasize that this scaling with the inverse of $\alpha$ is necessary information-theoretically \cite{santhanam2012information}. All current algorithms for learning MRFs and RBMs have this dependency, and it is impossible to remove it while still guaranteeing the recovery of the structure of the model.

Instead, in this section we give an algorithm that learns an RBM with small prediction error, independently of $\alpha$. We necessarily lose the guarantee on structure recovery, but we guarantee accurate prediction even for RBMs in which $\alpha$ is arbitrarily degenerate. The algorithm is composed of a structure learning step that recovers the MRF neighborhoods corresponding to large potentials, and a regression step that estimates the values of these potentials.

\begin{figure}
  \centering
  \includegraphics[scale=0.26]{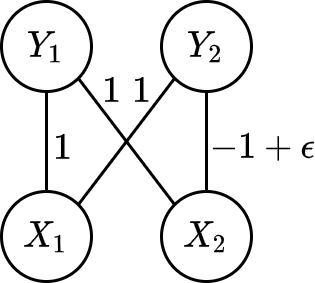}
  \caption{RBM with $\alpha \to 0$ as $\epsilon \to 0$ and $\alpha^*=1$, $\beta^*=2$ when $0 \leq \epsilon \leq 2$. The $X$ variables represent observed variables, the $Y$ variables represent latent variables, and the edges represent non-zero interactions between variables. All external field terms are zero.}
  \label{fig:cov_go_zero}
\end{figure}

\subsection{Structure learning algorithm}
\label{section:alpha_result}

The structure learning algorithm is guaranteed to recover the MRF neighborhoods corresponding to potentials that are at least $\zeta$ in absolute value. The guarantees of the algorithm are stated in Theorem \ref{thm:recov_alpha}, which is proved in the Appendix D.

The main differences between this algorithm and the one in Section \ref{sec:algorithm} are: first, the thresholds for $\hat{\nu}_{u,I|S}$ are defined in terms of $\zeta$ instead of $\alpha$, and second, the threshold for $\hat{\nu}_{u,I|S}$ in the additive step (Step 2) is smaller than that used in the pruning step (Step 3), in order to guarantee the pruning of all non-neighbors. The algorithm is described in detail in the Appendix C.

\begin{theorem}
\label{thm:recov_alpha}
Fix $\omega > 0$. Suppose we are given $M$ samples from an RBM, where $M$ is as in Theorem \ref{thm:algorithm} if $\alpha$ were equal to
\[\alpha = \frac{\zeta}{\sqrt{3} \cdot 2^{D/2+2^{s}} \cdot D^{2^{s-1}(2^s+2)}}.\]
Then with probability at least $1-\omega$, our algorithm, when run starting from each observed variable $u$, recovers a subset of the MRF neighbors of $u$, such that all neighbors which are connected to $u$ through a Fourier coefficient of absolute value at least $\zeta$ are included in the subset. Each run of the algorithm takes $O(MLn^{2^s+1})$ time.
\end{theorem}

\subsection{Regression algorithm}
\label{sec_prediction_error}
Note that 
\[\mathbb{P}(X_u=1|X_{[n] \setminus \{u\}}=x_{[n] \setminus \{u\}}) = \sigma\left(2 \sum_{S \subseteq [n] \setminus \{u\}} \hat{f}(S \cup \{u\}) \chi_S(x) \right).\]
Therefore, following the approach of \cite{wu2019sparse}, we can frame the recovery of the Fourier coefficients as a regression task. Let $n(u)$ be the set of MRF neighbors of $u$ recovered by the algorithm in Section \ref{section:alpha_result}. Note that $|n(u)| \leq D$. Let $z \in \{-1,1\}^{2^{|n(u)|}}$, $w \in \mathbb{R}^{2^{|n(u)|}}$, and $y \in \{-1,1\}$, with $z_S = \chi_S(X)$, $w_S = 2 \hat{f}(S \cup \{u\})$, and $y=X_u$, for all subsets $S \subseteq n(u)$. Then, if $n(u)$ were equal to the true set of MRF neighbors, we could rewrite the conditional probability statement above as
\[\mathbb{P}(y=1|z) = \sigma(w \cdot z), \quad \text{with } ||w||_1 \leq 2\gamma.\]
Then, finding an estimate $\hat{w}$ would amount to a constrained regression problem. In our setting, we solve the same problem, and we show that the resulting estimate has small prediction error. We estimate $\hat{w}$ as follows:
\[\hat{w} \in \argmin_{w \in \mathbb{R}^{|n(u)|}} \frac{1}{M} \sum_{i=1}^M l(y^{(i)} (w \cdot z^{(i)})) \quad \text{ s.t. } ||w||_1 \leq 2\gamma,\]
where we assume we have access to $M$ i.i.d. samples $(z, y)$, and where $l: \mathbb{R} \to \mathbb{R}$ is the loss function
\[l(y (w \cdot z)) = \ln(1 + e^{-y (w \cdot z)}) = \begin{cases}
-\ln \sigma(w \cdot z), & \text{ if } y = 1\\
-\ln (1 -\sigma(w \cdot z )), & \text{ if } y = -1
\end{cases}.\]
The objective above is convex, and the problem is solvable in time $\tilde{O}((2^D)^4)$ by the $l_1$-regularized logistic regression method described in \cite{koh2007interior}. Then, Theorem \ref{thm:alpha_distance} gives theoretical guarantees for the prediction error achieved by this regression algorithm. The proof is deferred to the Appendix D.

\begin{theorem}
\label{thm:alpha_distance}
Fix $\delta > 0$ and $\epsilon > 0$. Suppose that we are given neighborhoods $n(u)$ satisfying the guarantees of Theorem \ref{thm:recov_alpha} for each observed variable $u$. Suppose that we are given $M$ samples from the RBM, and that we have
\[M = \Omega\left(\gamma^2\ln(8 \cdot n \cdot 2^{D} / \delta)/\epsilon^2\right), \quad \zeta \leq \frac{\sqrt{\epsilon}}{D^d \sqrt{1+e^{2\gamma}}}.\]
Let $z_u$ and $\hat{w}_u$ be the features and the estimate of the weights when the regression algorithm is run at observed variable $u$. Then, with probability at least $1-\delta$, for all variables $u$,
\[\mathbb{E}\left[\left(\mathbb{P}(X_u=1|X_{\setminus u}) - \sigma\left(\hat{w}_u \cdot z_u\right)\right)^2\right] \leq \epsilon.\]
\end{theorem}

The sample complexity of the combination of structure learning and regression is given by the sum of the sample complexities of the two algorithms. When $\delta$ is constant, the number of samples required by regression is absorbed by the number of samples required by strucutre learning. For structure learning, plugging in the upper bound on $\zeta$ required by Theorem \ref{thm:alpha_distance}, we get that the sample complexity is exponential in $\epsilon^{-1}$. Note that the factors $D^d$ and $\sqrt{1+e^{2\gamma}}$ in the upper bound on $\zeta$, as well as the factors that appear in Theorem \ref{thm:recov_alpha} from the relative scaling of $\alpha$ and $\zeta$, do not influence the sample complexity much, because factors of similar order already appear in the sample complexity of the structure learning algorithm. Overall, for constant $\delta$ and constant $\epsilon$, the combined sample complexity is comparable to that of the algorithm in Section \ref{sec:algorithm}, without the $\alpha$ dependency.

\newpage

\bibliography{main.bib}

\begin{thebibliography}{10}

\bibitem{anandkumar2013learning}
Animashree Anandkumar, Ragupathyraj Valluvan, et~al.
\newblock Learning loopy graphical models with latent variables: Efficient
  methods and guarantees.
\newblock {\em The Annals of Statistics}, 41(2):401--435, 2013.

\bibitem{bresler2014structure}
Guy Bresler, David Gamarnik, and Devavrat Shah.
\newblock Structure learning of antiferromagnetic {I}sing models.
\newblock In {\em Advances in Neural Information Processing Systems}, pages
  2852--2860, 2014.

\bibitem{bresler2016learning}
Guy Bresler and Mina Karzand.
\newblock Learning a tree-structured {I}sing model in order to make
  predictions.
\newblock {\em Annals of Statistics}, 48(2):713--737, 2020.

\bibitem{bresler2019learning}
Guy Bresler, Frederic Koehler, and Ankur Moitra.
\newblock Learning restricted {B}oltzmann machines via influence maximization.
\newblock In {\em Proceedings of the 51st Annual ACM SIGACT Symposium on Theory
  of Computing}, pages 828--839. ACM, 2019.

\bibitem{brush1967history}
Stephen~G Brush.
\newblock History of the {L}enz-{I}sing model.
\newblock {\em Reviews of modern physics}, 39(4):883, 1967.

\bibitem{chandrasekaran2010latent}
Venkat Chandrasekaran, Pablo~A Parrilo, and Alan~S Willsky.
\newblock Latent variable graphical model selection via convex optimization.
\newblock In {\em 2010 48th Annual Allerton Conference on Communication,
  Control, and Computing (Allerton)}, pages 1610--1613. IEEE, 2010.

\bibitem{choi2011learning}
Myung~Jin Choi, Vincent~YF Tan, Animashree Anandkumar, and Alan~S Willsky.
\newblock Learning latent tree graphical models.
\newblock {\em Journal of Machine Learning Research}, 12(May):1771--1812, 2011.

\bibitem{deng2002prediction}
Minghua Deng, Kui Zhang, Shipra Mehta, Ting Chen, and Fengzhu Sun.
\newblock Prediction of protein function using protein-protein interaction
  data.
\newblock In {\em Proceedings. IEEE Computer Society Bioinformatics
  Conference}, pages 197--206. IEEE, 2002.

\bibitem{freund1992unsupervised}
Yoav Freund and David Haussler.
\newblock Unsupervised learning of distributions on binary vectors using two
  layer networks.
\newblock In {\em Advances in neural information processing systems}, pages
  912--919, 1992.

\bibitem{gotze2019higher}
Friedrich G{\"o}tze, Holger Sambale, Arthur Sinulis, et~al.
\newblock Higher order concentration for functions of weakly dependent random
  variables.
\newblock {\em Electronic Journal of Probability}, 24, 2019.

\bibitem{hamilton2017information}
Linus Hamilton, Frederic Koehler, and Ankur Moitra.
\newblock Information theoretic properties of {M}arkov random fields, and their
  algorithmic applications.
\newblock In {\em Advances in Neural Information Processing Systems}, pages
  2463--2472, 2017.

\bibitem{hammersley1971markov}
John~M Hammersley and Peter Clifford.
\newblock {M}arkov fields on finite graphs and lattices.
\newblock {\em Unpublished manuscript}, 46, 1971.

\bibitem{hinton2002training}
Geoffrey~E Hinton.
\newblock Training products of experts by minimizing contrastive divergence.
\newblock {\em Neural computation}, 14(8):1771--1800, 2002.

\bibitem{hinton2009replicated}
Geoffrey~E Hinton and Russ~R Salakhutdinov.
\newblock Replicated softmax: an undirected topic model.
\newblock In {\em Advances in neural information processing systems}, pages
  1607--1614, 2009.

\bibitem{jaitly2011learning}
Navdeep Jaitly and Geoffrey Hinton.
\newblock Learning a better representation of speech soundwaves using
  restricted {B}oltzmann machines.
\newblock In {\em 2011 IEEE International Conference on Acoustics, Speech and
  Signal Processing (ICASSP)}, pages 5884--5887. IEEE, 2011.

\bibitem{kearns1998efficient}
Michael Kearns.
\newblock Efficient noise-tolerant learning from statistical queries.
\newblock {\em Journal of the ACM (JACM)}, 45(6):983--1006, 1998.

\bibitem{klivans2017learning}
Adam Klivans and Raghu Meka.
\newblock Learning graphical models using multiplicative weights.
\newblock In {\em 2017 IEEE 58th Annual Symposium on Foundations of Computer
  Science (FOCS)}, pages 343--354. IEEE, 2017.

\bibitem{koh2007interior}
Kwangmoo Koh, Seung-Jean Kim, and Stephen Boyd.
\newblock An interior-point method for large-scale $\ell_1$-regularized
  logistic regression.
\newblock {\em Journal of Machine learning research}, 8(Jul):1519--1555, 2007.

\bibitem{li2012markov}
Stan~Z Li.
\newblock {\em {M}arkov random field modeling in computer vision}.
\newblock Springer Science \& Business Media, 2012.

\bibitem{mossel2005learning}
Elchanan Mossel and S{\'e}bastien Roch.
\newblock Learning nonsingular phylogenies and hidden {M}arkov models.
\newblock In {\em Proceedings of the thirty-seventh annual ACM symposium on
  Theory of computing}, pages 366--375, 2005.

\bibitem{nomura2017restricted}
Yusuke Nomura, Andrew~S Darmawan, Youhei Yamaji, and Masatoshi Imada.
\newblock Restricted {B}oltzmann machine learning for solving strongly
  correlated quantum systems.
\newblock {\em Physical Review B}, 96(20):205152, 2017.

\bibitem{o2014analysis}
Ryan O'Donnell.
\newblock {\em Analysis of boolean functions}.
\newblock Cambridge University Press, 2014.

\bibitem{salakhutdinov2007restricted}
Ruslan Salakhutdinov, Andriy Mnih, and Geoffrey Hinton.
\newblock Restricted {B}oltzmann machines for collaborative filtering.
\newblock In {\em Proceedings of the 24th international conference on Machine
  learning}, pages 791--798. ACM, 2007.

\bibitem{santhanam2012information}
Narayana~P Santhanam and Martin~J Wainwright.
\newblock Information-theoretic limits of selecting binary graphical models in
  high dimensions.
\newblock {\em IEEE Transactions on Information Theory}, 58(7):4117--4134,
  2012.

\bibitem{smolensky1986information}
Paul Smolensky.
\newblock Information processing in dynamical systems: Foundations of harmony
  theory.
\newblock Technical report, Colorado Univ at Boulder Dept of Computer Science,
  1986.

\bibitem{vuffray2019efficient}
Marc Vuffray, Sidhant Misra, and Andrey~Y Lokhov.
\newblock Efficient learning of discrete graphical models.
\newblock {\em arXiv preprint arXiv:1902.00600}, 2019.

\bibitem{wei2007markov}
Zhi Wei and Hongzhe Li.
\newblock A {M}arkov random field model for network-based analysis of genomic
  data.
\newblock {\em Bioinformatics}, 23(12):1537--1544, 2007.

\bibitem{wu2019sparse}
Shanshan Wu, Sujay Sanghavi, and Alexandros~G Dimakis.
\newblock Sparse logistic regression learns all discrete pairwise graphical
  models.
\newblock In {\em Advances in Neural Information Processing Systems}, pages
  8069--8079, 2019.

\bibitem{yan2015restricted}
Yan Yan, Xinbing Qin, Yige Wu, Nannan Zhang, Jianping Fan, and Lei Wang.
\newblock A restricted {B}oltzmann machine based two-lead electrocardiography
  classification.
\newblock In {\em 2015 IEEE 12th international conference on wearable and
  implantable body sensor networks (BSN)}, pages 1--9. IEEE, 2015.

\end{thebibliography}
\bibliographystyle{plain}

\newpage
\raggedbottom
\allowdisplaybreaks
\appendix
\section{Proof of Theorem \ref{thm:nu_bound_alone}}

\subsection{Proof of Lemma \ref{lemma:expr_nu}}
\label{app_proof_expr_nu_full}

We first state and prove Lemmas \ref{lemma_uniform_e_fx}, \ref{lemma_prob_expr_nu}, and \ref{lemma:expr_nu_full}, which provide the foundation for the proof of Lemma \ref{lemma:expr_nu}.

\begin{lemma}
\label{lemma_uniform_e_fx}
Let $f(x) = \sum_{j=1}^m \rho(J_j \cdot x + g_j) + h^T x$, where $x \in \{-1,1\}^n$, $J \in \mathbb{R}^{n \times m}$, $h \in \mathbb{R}^n$, and $g \in \mathbb{R}^m$. Then
\begin{align*}
&\mathbb{E}_\mathcal{U}[\mathbbm{1}_{X_I=x_I} e^{f(x)} | X_S=x_S]\\
&\quad = \sum_{q\in\{-1,1\}^m} e^{g \cdot q} \prod_{i \in S} e^{x_i(J^{(i)} \cdot q + h_i)} \prod_{i \in [n] \setminus S} \cosh(J^{(i)} \cdot q + h_i) \prod_{i \in I} \sigma(2x_i(J^{(i)}\cdot q + h_i))
\end{align*}
where $\mathcal{U}$ denotes the uniform distribution over $\{-1,1\}^n$ and where $J^{(i)}$ denotes the $i$-th row of $J$.
\end{lemma}
\begin{proof}
\begin{align*}
&\mathbb{E}_{\mathcal{U}}[\mathbbm{1}_{X_I=x_I} e^{f(x)}|X_S=x_S]\\
&\quad = \frac{1}{2^{n-|S|}} \sum_{x \in \{-1,1\}^{n}} \mathbbm{1}_{X_S=x_S, X_I=x_I} \cdot e^{h \cdot x} \prod_{j=1}^m \left(e^{J_j \cdot x + g_j} + e^{-J_j \cdot x - g_j}\right) \\
&\quad = \frac{1}{2^{n-|S|}} \sum_{x \in \{-1,1\}^n} \mathbbm{1}_{X_S=x_S, X_I=x_I} \cdot e^{h \cdot x} \sum_{q \in \{-1, 1\}^m} e^{(x^T J + g^T) q}\\
&\quad = \frac{1}{2^{n-|S|}} \sum_{x \in \{-1,1\}^n} \mathbbm{1}_{X_S=x_S, X_I=x_I} \sum_{q \in \{-1, 1\}^m} e^{x^T (J q+h)} e^{g \cdot q}\\
&\quad = \frac{1}{2^{n-|S|}} \sum_{q \in \{-1, 1\}^m} \sum_{x \in \{-1,1\}^n} \mathbbm{1}_{X_S=x_S, X_I=x_I} \cdot e^{x^T (Jq+h)}e^{g \cdot q}\\
&\quad = \frac{1}{2^{n-|S|}} \sum_{q \in \{-1, 1\}^m} e^{g \cdot q} \sum_{x \in \{-1,1\}^n} \mathbbm{1}_{X_S=x_S, X_I=x_I} \cdot e^{\sum_{i=1}^n x_i (J^{(i)} \cdot q + h_i)}\\
&\quad = \frac{1}{2^{n-|S|}} \sum_{q \in \{-1, 1\}^m} e^{g \cdot q} \left(\sum_{x_{[n]\sm (S \cup I)}} \prod_{i \in [n]\sm (S \cup I)} e^{x_i (J^{(i)} \cdot q + h_i)} \right)\prod_{i \in S \cup I} e^{x_i (J^{(i)} \cdot q + h_i)}\\
&\quad = \frac{1}{2^{n-|S|}} \sum_{q \in \{-1, 1\}^m} e^{g \cdot q} \prod_{i \in [n]\sm (S \cup I)} \left(e^{J^{(i)} \cdot q + h_i} + e^{-J^{(i)} \cdot q - h_i}\right) \prod_{i \in S \cup I} e^{x_i (J^{(i)} \cdot q + h_i)}\\
&\quad = \sum_{q \in \{-1, 1\}^m} e^{g \cdot q} \prod_{i \in S}e^{x_i (J^{(i)} \cdot q + h_i)} \prod_{i \in [n]\sm (S \cup I)} \cosh(J^{(i)} \cdot q + h_i) \prod_{i \in I} \frac{e^{x_i (J^{(i)} \cdot q + h_i)}}{2}\\
&\quad = \sum_{q \in \{-1, 1\}^m} e^{g \cdot q} \prod_{i \in S}e^{x_i (J^{(i)} \cdot q + h_i)} \prod_{i \in [n]\sm S} \cosh(J^{(i)} \cdot q + h_i) \prod_{i \in I} \frac{e^{x_i (J^{(i)} \cdot q + h_i)}}{2\cosh(J^{(i)} \cdot q + h_i)}\\
&\quad = \sum_{q \in \{-1, 1\}^m} e^{g \cdot q} \prod_{i \in S}e^{x_i (J^{(i)} \cdot q + h_i)} \prod_{i \in [n]\sm S} \cosh(J^{(i)} \cdot q + h_i) \prod_{i \in I} \sigma(2x_i(J^{(i)}\cdot q + h_i)).
\end{align*}
\end{proof}

\begin{lemma}
\label{lemma_prob_expr_nu}
Fix subsets of observed variables $I$ and $S$, such that the two are disjoint. Then
\[\mathbb{P}(X_I=x_I|X_S=x_S) = \sum_{q \in \{-1,1\}^m} \lambda(q, x_S) \prod_{i \in I} \sigma(2x_i(J^{(i)}\cdot q + h_i))\]
where 
\[\lambda(q, x_S) = \frac{e^{g \cdot q} \prod_{i \in S}e^{x_i (J^{(i)} \cdot q + h_i)} \prod_{i \in [n]\sm S} \cosh(J^{(i)} \cdot q + h_i)}{\sum_{q' \in \{-1, 1\}^m} e^{g \cdot q'} \prod_{i \in S}e^{x_i (J^{(i)} \cdot q' + h_i)} \prod_{i \in [n]\sm S} \cosh(J^{(i)} \cdot q' + h_i)}.\]
\end{lemma}
\begin{proof}
The MRF of the observed variables has a probability mass function that is proportional to $\exp(f(x))$, where $f(x)$ is as in Lemma \ref{lemma_uniform_e_fx}. Then
\begin{align*}
&\mathbb{P}(X_I=x_I|X_S=x_S)\\
& = \frac{\mathbb{P}(X_S=x_S, X_I=x_I)}{\mathbb{P}(X_S=x_S)}\\
&= \frac{\mathbb{E}[\mathbbm{1}_{X_S=x_S,X_I=x_I}]}{\mathbb{E}[\mathbbm{1}_{X_S=x_S}]}\\
&= \frac{\frac{1}{Z}\sum_{x\in\{-1,1\}^n} \mathbbm{1}_{X_S=x_S,X_I=x_I} \cdot e^{f(x)}}{\frac{1}{Z}\sum_{x\in\{-1,1\}^n} \mathbbm{1}_{X_S=x_S} \cdot e^{f(x)}}\\
&= \frac{\frac{1}{2^{n-|S|}}\sum_{x\in\{-1,1\}^n} \mathbbm{1}_{X_S=x_S,X_I=x_I} \cdot e^{f(x)}}{\frac{1}{2^{n-|S|}}\sum_{x\in\{-1,1\}^n} \mathbbm{1}_{X_S=x_S} \cdot e^{f(x)}}\\
&= \frac{\mathbb{E}_\mathcal{U}[\mathbbm{1}_{X_I=x_I} e^{f(x)}|X_S=x_S]}{\mathbb{E}_\mathcal{U}[e^{f(x)}|X_S=x_S]}\\
&= \frac{\sum_{q \in \{-1, 1\}^m} e^{g \cdot q} \prod_{i \in S}e^{x_i (J^{(i)} \cdot q + h_i)} \prod_{i \in [n] \setminus S} \cosh(J^{(i)} \cdot q + h_i) \prod_{i \in I} \sigma(2x_i(J^{(i)}\cdot q + h_i))}{\sum_{q' \in \{-1, 1\}^m} e^{g \cdot q'} \prod_{i \in S}e^{x_i (J^{(i)} \cdot q' + h_i)} \prod_{i \in [n] \setminus S}  \cosh(J^{(i)} \cdot q' + h_i)}\\
&= \sum_{q \in \{-1,1\}^m} \lambda(q, x_S) \prod_{i \in I} \sigma(2x_i(J^{(i)}\cdot q + h_i)).
\end{align*}
\end{proof}

\begin{lemma}
\label{lemma:expr_nu_full}
Fix observed variable $u$ and subsets of observed variables $I$ and $S$, such that all three are disjoint. Then
\[\nu_{u,I|S}(x_{u},x_{I}|x_S) = \left| \sum_{q \in \{-1,1\}^m} \bar{f}(q, x_u, x_S) \prod_{i \in I} \sigma(2x_i(J^{(i)}\cdot q + h_i)) \right|\]
where $J^{(i)}$ denotes the $i$-th row of $J$ and where
\[\bar{f}(q, x_u, x_S) = \lambda(q, x_S) \left[\sigma(2x_u(J^{(u)}\cdot q + h_u)) - \mathbb{E}_{q' \sim \lambda(\cdot, x_S)} \sigma(2x_u(J^{(u)}\cdot q' + h_u)) \right],\]
\[\lambda(q, x_S) = \frac{e^{g \cdot q} \prod_{i \in S}e^{x_i (J^{(i)} \cdot q + h_i)} \prod_{i \in [n]\sm S} \cosh(J^{(i)} \cdot q + h_i)}{\sum_{q' \in \{-1, 1\}^m} e^{g \cdot q'} \prod_{i \in S}e^{x_i (J^{(i)} \cdot q' + h_i)} \prod_{i \in [n]\sm S} \cosh(J^{(i)} \cdot q' + h_i)}.\]
\end{lemma}
\begin{proof}
We apply Lemma \ref{lemma_prob_expr_nu} to the terms in the definition of $\nu_{u,I|S}(x_u,x_I|x_S)$:
\begin{align*}
&\mathbb{P}(X_{u}=x_{u}, X_{I}=x_{I}|X_S=x_S) - \mathbb{P}(X_{u}=x_{u}|X_S=x_S)\mathbb{P}(X_{I}=x_{I}|X_S=x_S)\\
& = \sum_{q\in\{-1,1\}^m} \lambda(q, x_S) \sigma(2x_u(J^{(u)}\cdot q + h_u)) \prod_{i \in I} \sigma(2x_i(J^{(i)}\cdot q + h_i))\\
& \quad - \left[\sum_{q\in\{-1,1\}^m} \lambda(q, x_S) \sigma(2x_u(J^{(u)}\cdot q + h_u))\right]\left[\sum_{q\in\{-1,1\}^m} \lambda(q, x_S) \prod_{i \in I} \sigma(2x_i(J^{(i)}\cdot q + h_i))\right]\\
& = \sum_{q\in\{-1,1\}^m} \sum_{q'\in\{-1,1\}^m} \lambda(q, x_S) \lambda(q', x_S) \sigma(2x_u(J^{(u)}\cdot q + h_u)) \prod_{i \in I} \sigma(2x_i(J^{(i)}\cdot q + h_i))\\
& \quad - \sum_{q\in\{-1,1\}^m} \sum_{q'\in\{-1,1\}^m} \lambda(q, x_S) \lambda(q', x_S) \sigma(2x_u(J^{(u)}\cdot q' + h_u)) \prod_{i \in I} \sigma(2x_i(J^{(i)}\cdot q + h_i))\\
& = \sum_{q\in\{-1,1\}^m} \lambda(q, x_S) \left[\sigma(2x_u(J^{(u)}\cdot q + h_u)) - \mathbb{E}_{q' \sim \lambda(\cdot, x_S)} \sigma(2x_u(J^{(u)}\cdot q' + h_u))\right]\\
&\qquad \cdot \prod_{i \in I} \sigma(2x_i(J^{(i)}\cdot q + h_i))\\
&=  \sum_{q \in \{-1,1\}^m} \bar{f}(q, x_u, x_S) \prod_{i \in I} \sigma(2x_i(J^{(i)}\cdot q + h_i)).
\end{align*}
\end{proof}

\begin{proof}[Proof of Lemma \ref{lemma:expr_nu}]
Note that, if $J_{i,j} = 0$ for all $i \in I$, then the term $\prod_{i \in I} \sigma(2x_i(J^{(i)}\cdot q + h_i))$ is independent of the value of $q_j$. Let $U = \{j \in [m]: J_{i,j} \neq 0 \text{ for some } i \in I\}$ be the set of latent variables with connections to observed variables in $I$. By Lemma \ref{lemma:expr_nu_full}, we can write then
\[\nu_{u,I|S}(x_{u},x_{I}|x_S)= \left| \sum_{q_U \in \{-1,1\}^{|U|}} \left( \sum_{q_{\sim U} \in \{-1,1\}^{m-|U|}} \bar{f}(q, x_u, x_S)\right) \prod_{i \in I} \sigma(2x_i(J^{(i)}\cdot q + h_i)) \right|.\]
\end{proof}

\subsection{Proof of Lemma \ref{lemma_non_cancel_abstract}}

A special case of Lemma \ref{lemma_non_cancel_abstract} is given in Lemma \ref{lemma_non_cancel_abstract_plusone}. Then, we prove Lemma \ref{lemma_non_cancel_abstract}. Lastly, Section \ref{subsec:impossibility} shows that these lemmas are tight in the size of the subset that they guarantee not to be cancelled.

\begin{lemma}
\label{lemma_non_cancel_abstract_plusone}
Let $x_{1,1}, ..., x_{m,m+1} \in [-1, 1]$. Then, for any $a \in \mathbb{R}^m$, there exists a subset $S \subseteq [m+1]$ with $|S| \leq m$ such that
\[\left| \sum_{i=1}^m a_i \prod_{j \in S} x_{i,j} \right| \geq \frac{1}{2^m - 1} \cdot \left| \sum_{i=1}^m a_i \prod_{j=1}^{m+1} x_{i,j} \right|.\]
\end{lemma}

\begin{proof}
We prove the claim by induction on $m$.

\textbf{Base case:} For $m=1$, we have
\[|a x_{1,1}| = \frac{|a x_{1,1} x_{1,2}|}{|x_{1,2}|} \geq |a x_{1,1} x_{1,2}|\]
\[|a x_{1,2}| = \frac{|a x_{1,1} x_{1,2}|}{|x_{1,1}|} \geq |a x_{1,1} x_{1,2}|\]
Therefore, the claim holds, with either $S=\{1\}$ or $S=\{2\}$. Note that if any of $a$, $x_{1,1}$, or $x_{1,2}$ is zero, then $a x_{1,1} x_{1,2} = 0$ and the claim holds trivially.

\textbf{Induction step:} Assume the claim holds for $m-1$.

Suppose $|\sum_{i=1}^m a_i \prod_{j \in S} x_{i,j}| < \frac{1}{2^m - 1} \cdot \left| \sum_{i=1}^m a_i \prod_{j=1}^{m+1} x_{i,j} \right|$ for any $S \subseteq [m]$; otherwise the induction step follows. By the triangle inequality, we have
\begin{align*}
\left|\sum_{i=1}^m a_i \prod_{j=1}^{m+1} x_{i,j}\right|
&= \left|\sum_{i=1}^m a_i x_{i,m+1} \prod_{j=1}^{m} x_{i,j}\right|\\
&\leq \left|\sum_{i=1}^m a_i (x_{i,m+1} - x_{m,m+1}) \prod_{j=1}^m x_{i,j}\right| + |x_{m,m+1}| \cdot \left|\sum_{i=1}^m a_i \prod_{j=1}^m x_{i,j}\right|.
\end{align*}
For the first term on the right-hand side, we clearly have $x_{i,m+1} - x_{m,m+1} = 0$ at $i=m$. Therefore, the term is of the form $\sum_{i=1}^{m-1} b_i \sum_{j=1}^m y_{i,j}$, so we can apply the inductive claim for $m-1$. Therefore, there exists a subset $S^* \subseteq [m]$ with $|S^*| \leq m-1$ such that
\[\left| \sum_{i=1}^m a_i (x_{i,m+1}-x_{m,m+1}) \prod_{j \in S^*} x_{i,j} \right| \geq \frac{1}{2^{m-1} - 1} \cdot \left| \sum_{i=1}^m a_i (x_{i,m+1} - x_{m,m+1}) \prod_{j=1}^m x_{i,j} \right|.\]
Overall, we get then the inequality:
\begin{align*}
&\left|\sum_{i=1}^m a_i \prod_{j=1}^{m+1} x_{i,j}\right|\\
&\quad \leq (2^{m-1}-1) \cdot \left| \sum_{i=1}^m a_i (x_{i,m+1}-x_{m,m+1}) \prod_{j \in S^*} x_{i,j} \right| + |x_{m,m+1}| \cdot \left|\sum_{i=1}^m a_i \prod_{j=1}^m x_{i,j}\right|\\
&\quad \leq (2^{m-1}-1) \cdot \left| \sum_{i=1}^m a_i x_{i,m+1} \prod_{j \in S^*} x_{i,j} \right| + (2^{m-1}-1) \cdot |x_{m,m+1}| \cdot \left| \sum_{i=1}^m a_i \prod_{j \in S^*} x_{i,j} \right|\\
&\quad\quad + |x_{m,m+1}| \cdot \left|\sum_{i=1}^m a_i \prod_{j=1}^m x_{i,j}\right|\\
&\quad \leq (2^{m-1}-1) \cdot \left| \sum_{i=1}^m a_i x_{i,m+1} \prod_{j \in S^*} x_{i,j} \right| + \left(\frac{2^{m-1}-1}{2^m-1} + \frac{1}{2^m-1} \right) \cdot \left| \sum_{i=1}^m a_i \prod_{j=1}^{m+1} x_{i,j} \right|
\end{align*}
where in the last inequality we used that $|x_{m,m+1}| \leq 1$ and our supposition that for all $S \subseteq [m]$, $|\sum_{i=1}^m a_i \prod_{j \in S} x_{i,j}| < \frac{1}{2^m - 1} \cdot \left| \sum_{i=1}^m a_i \prod_{j=1}^{m+1} x_{i,j} \right|$. Then, reordering:
\begin{align*}
\left| \sum_{i=1}^m a_i x_{i,m+1} \prod_{j \in S^*} x_{i,j} \right|
&\geq \frac{1 - \frac{2^{m-1}-1}{2^m-1} - \frac{1}{2^{m}-1}}{2^{m-1}-1} \left|\sum_{i=1}^m a_i \prod_{j=1}^{m+1} x_{i,j}\right| = \frac{1}{2^m-1} \left|\sum_{i=1}^m a_i \prod_{j=1}^{m+1} x_{i,j}\right|.
\end{align*}
Then, in this case, $S^* \cup \{m+1\}$ is the desired subset. Note that we selected $S^*$ such that $|S^*| \leq m-1$, so $|S^* \cup \{m+1\}| \leq m$.
\end{proof}

\begin{proof}[Proof of Lemma \ref{lemma_non_cancel_abstract}]
Partition $[n]$ into $m+1$ subsets $Q_1 = [1, \lceil\frac{n}{m+1}\rceil], Q_2 = [\lceil\frac{n}{m+1}\rceil + 1, 2 \lceil\frac{n}{m+1}\rceil], ..., Q_{m+1} = [m \lceil\frac{n}{m+1}\rceil + 1, n]$. Then, apply Lemma \ref{lemma_non_cancel_abstract_plusone} to
\[\left| \sum_{i=1}^m a_i \prod_{j=1}^{m+1} \left(\prod_{k \in Q_j} x_{i,k}\right) \right|\]
where we know that $\prod_{k \in Q_j} x_{i,k} \in [-1,1]$, for all $j$. Then, there exists a subset $S \subseteq [m+1]$ with $|S| \leq m$ such that
\[\left| \sum_{i=1}^m a_i \prod_{j \in S} \left(\prod_{k \in Q_j} x_{i,k}\right) \right| \geq \frac{1}{2^m-1} \left| \sum_{i=1}^m a_i \prod_{j=1}^{m+1} \left(\prod_{k \in Q_j} x_{i,k}\right) \right|.\]
Let $S' = \bigcup_{j \in S} Q_j$. Then $S' \subseteq [n]$ with $|S'| \leq n - \lfloor \frac{n}{m+1} \rfloor \leq n - \frac{n}{m+1} + 1$, and 
\[\left| \sum_{i=1}^m a_i \prod_{j \in S'} x_{i,j} \right| \geq \frac{1}{2^m-1} \left| \sum_{i=1}^m a_i \prod_{j=1}^{n} x_{i,j} \right|.\]

Now, if $|S'| > m$, apply the same technique recursively to $S'$: partition it into $m+1$ equal subsets and apply Lemma \ref{lemma_non_cancel_abstract_plusone}. Continue until you obtain a subset of size at most $m$.

We now bound the number of iterations required. Let $n_t$ be the size of the set at timestep $t$ (at the beginning, $n_0 = n$). We have
\begin{align*}
n_{t}
&\leq n_{t-1} \left(1 - \frac{1}{m+1}\right) + 1\\
&\leq n_{t-2} \left(1 - \frac{1}{m+1}\right)^2 + \left(1 - \frac{1}{m+1}\right) + 1\\
&\leq ...\\
&\leq n \left(1 - \frac{1}{m+1}\right)^t + \sum_{q=0}^{t-1} \left(1 - \frac{1}{m+1}\right)^q\\
&\leq n \left(1 - \frac{1}{m+1}\right)^t + m+1.
\end{align*}
Let $T$ be the smallest timestep such that $n_T < m + 2$. An upper bound on $T$ is obtained as
\begin{align*}
n \left(1 - \frac{1}{m+1}\right)^T < 1
&\Longrightarrow n e^{-2T/(m+1)} < 1\\
&\Longrightarrow T > \frac{m+1}{2}\ln(n)
\end{align*}
where we used that $e^{-2x} \leq 1-x$ for $0 \leq x \leq 1/2$. Because $T$ is an integer, the correct upper bound is $\frac{m+1}{2}\ln(n) + 1$. Then, at this step, we are guaranteed that $n_T \leq m+1$. One more step may be required to go from size $m+1$ to size $m$. Therefore, an upper bound on the number of steps until $n_t \leq m$ is $\frac{m+1}{2}\ln(n)+2$. Then, the factor due to applications of Lemma \ref{lemma_non_cancel_abstract_plusone} is
\begin{align*}
\left(\frac{1}{2^m - 1}\right)^{\frac{m+1}{2}\ln(n)+2}
&\geq \left(\frac{1}{2^m}\right)^{\frac{m+1}{2}\ln(n)+2}
= \frac{1}{2^{m(m+1)/2 \ln(n)+2m}}\\
&= \frac{1}{4^m}\left(\frac{1}{n}\right)^{m(m+1)\log_2(e)/2}
\geq \frac{1}{4^m} \left(\frac{1}{n}\right)^{m(m+1)}.
\end{align*}
\end{proof}

\subsubsection{Tightness of non-cancellation result}
\label{subsec:impossibility}

Lemma \ref{lemma_noncanc_imposs} shows that, in the setting of Lemmas \ref{lemma_non_cancel_abstract_plusone} and \ref{lemma_non_cancel_abstract}, it is possible for all subsets of size strictly less than $m$ to be completely cancelled. Therefore, the guarantee on the existence of a subset of size at most $m$ that is non-cancelled is tight. 

We emphasize that, for the RBM setting, this result does not imply an impossibility of finding subsets of size less than $2^s$ with non-zero mutual information proxy. One reason for this is that, in the RBM setting, the terms of the sums that we are interested in have additional constraints which are not captured by the general setting of this section.

\begin{lemma}
\label{lemma_noncanc_imposs}
For any $c \in \mathbb{R}$, there exists some $x_{1,1}, ..., x_{m,m} \in [-1, 1]$ and some $a \in \mathbb{R}^m$ such that 
\[\left| \sum_{i=1}^m a_i \prod_{j=1}^{m} x_{i,j} \right| = c\]
and for any subset $S \subseteq [m]$ with $|S| \leq m-1$
\[\left| \sum_{i=1}^m a_i \prod_{j \in S} x_{i,j} \right| = 0.\]
\end{lemma}
\begin{proof}
Let $x_{1,1} = ... = x_{1,m} = x_1$, ..., $x_{m,1} = ... = x_{m,m} = x_m$. Then we want to select some $x_1, ..., x_m \in [-1, 1]$ and some $a \in \mathbb{R}^m$ such that
\[\left[\begin{array}{cccc}
x_1 & x_2 & \cdots & x_m\\
\vdots & \vdots & \ddots & \vdots\\
x_1^{m-1} & x_2^{m-1} & \cdots & x_m^{m-1}\\
x_1^m & x_2^m & \cdots & x_m^m
\end{array}\right] \left[\begin{array}{c}
a_1\\
\vdots\\
a_{m-1}\\
a_m
\end{array}\right] = \left[\begin{array}{c}
0\\
\vdots\\
0\\
c
\end{array}\right].\]
Select some arbitrary $x_1, ..., x_m \in [-1, 1]$ such that the matrix on the left-hand-side has full rank. Note that $(x, ..., x^{m-1}, x^m)$ for $x \in \mathbb{R}$ is a point on the moment curve, and it is known that any such $m$ distinct non-zero points are linearly independent. Therefore, any distinct non-zero $x_1, ..., x_m \in [-1, 1]$ will do. Then, by matrix inversion, there exists some $a \in \mathbb{R}^m$ such that the relation holds.
\end{proof}

\subsection{Proof of Lemma \ref{lemma_non_cancel_nu_indiv_cond}}

\begin{proof}[Proof of Lemma \ref{lemma_non_cancel_nu_indiv_cond}]
Apply Lemma \ref{lemma_non_cancel_abstract} to the form of $\nu_{u, I|S}(x_u, x_{I}|x_S)$ in Lemma \ref{lemma:expr_nu}, with
\[\sum_{q_{\sim U} \in \{-1,1\}^{m - |U|}} \bar{f}(q, x_u, x_S)\]
treated as a coefficient (i.e., $a$ in Lemma \ref{lemma_non_cancel_abstract}) and
\[\sigma(2x_i(J^{(i)}\cdot q + h_i))\]
treated as a variable in $[-1, 1]$ (i.e., $x$ in Lemma \ref{lemma_non_cancel_abstract}). Then there exists a subset $I' \subseteq I$ with $|I'| \leq 2^{|U|}$ such that 
\begin{align*}
&\left| \sum_{q_U \in \{-1,1\}^{|U|}} \left( \sum_{q_{\sim U} \in \{-1,1\}^{m-|U|}} \bar{f}(q, x_u, x_S)\right) \prod_{i \in I'} \sigma(2x_i(J^{(i)}\cdot q + h_i)) \right|\\
&\qquad \geq \frac{1}{4^{2^{|U|}}} \left(\frac{1}{|I|}\right)^{2^{|U|}(2^{|U|}+1)} \nu_{u, I|S}(x_u, x_{I}|x_S).
\end{align*}
Note that the latent variables connected to observed variables in $I'$ are a subset of $U$. Then, by Lemma \ref{lemma:expr_nu}, the expression on the left-hand side is equal to $\nu_{u,I'|S}(x_u,x_{I'}|x_S)$ where $x_{I'}$ agrees with $x_I$. Finally, note that $|U| \leq s$.
\end{proof}

\subsection{Proof of Lemma \ref{lemma_non_cancel_nu_cond_full}}

\begin{proof}[Proof of Lemma \ref{lemma_non_cancel_nu_cond_full}]
We have that 
\[\nu_{u,I|S} = \sum_{x_u, x_I, x_S} \frac{\mathbb{P}(X_S=x_S)}{2^{|I|+1}} \nu_{u,I|S}(x_u,x_I|x_S).\]
Hence, $\nu_{u,I|S}$ is a sum of $2^{|S|+|I|+1}$ terms $\nu_{u,I|S}(x_u,x_I|x_S)$. Lemma \ref{lemma_non_cancel_nu_indiv_cond} applies to each term $\nu_{u,I|S}(x_u,x_I|x_S)$ individually. However, the subset $I'$ with $|I'| \leq 2^s$ that is guaranteed to exist by Lemma \ref{lemma_non_cancel_nu_indiv_cond} may be a function of the specific assignment $x_u$, $x_I$, and $x_S$. Let $I^*(x_u,x_I|x_S)$ be the subset $I'$ with $|I'| \leq 2^s$ that is guaranteed to exist by Lemma \ref{lemma_non_cancel_nu_indiv_cond} for assignment $x_u$, $x_I$, and $x_S$.

The number of non-empty subsets $I' \subseteq I$ with $|I'| \leq 2^s$ is at most $|I|^{2^s}$. Then, by the pigeonhole principle, there exists some $I' \subseteq I$ with $|I'| \leq 2^s$ which captures at least $\frac{1}{|I|^{2^{s}}}$ of the total mass of $\nu_{u,I|S}$:
\[\sum_{\substack{x_u,x_I,x_S\\I^*(x_u,x_I|x_S)=I'}} \frac{\mathbb{P}(X_S=x_S)}{2^{|I|+1}} \nu_{u,I|S}(x_u,x_I|x_S) \geq \frac{1}{|I|^{2^s}} \nu_{u,I|S}.\]
Applying Lemma \ref{lemma_non_cancel_nu_indiv_cond} to each of the terms $\nu_{u,I|S}(x_u,x_I|x_S)$ that we sum over on the left-hand side, we get 
\[\sum_{\substack{x_u,x_I,x_S\\I^*(x_u,x_I|x_S)=I'}} \frac{\mathbb{P}(X_S=x_S)}{2^{|I|+1}} \nu_{u,I'|S}(x_u,x_{I'}|x_S) \geq \frac{1}{(4|I|)^{2^s}} \left( \frac{1}{|I|} \right)^{2^s(2^s+1)} \nu_{u,I|S}.\]
Note that we also have
\begin{align*}
\nu_{u,I'|S}
&= \sum_{x_u, x_{I'},x_S} \frac{\mathbb{P}(X_S=x_S)}{2^{|I'|+1}} \nu_{u, I'|S}(x_u,x_{I'}|x_S)\\
&= 2^{|I|-|I'|} \sum_{x_u, x_{I'},x_S} \frac{\mathbb{P}(X_S=x_S)}{2^{|I|+1}} \nu_{u, I'|S}(x_u,x_{I'}|x_S)\\
&\geq \sum_{\substack{x_u,x_I,x_S\\I^*(x_u,x_I|x_S)=I'}} \frac{\mathbb{P}(X_S=x_S)}{2^{|I|+1}} \nu_{u,I'|S}(x_u,x_{I'},x_S).
\end{align*}
The inequality step above holds because, for each assignment $x_{I'}$, there are $2^{|I|-|I'|}$ assignments $x_I$ that are in accord with it. Hence, each term $\nu_{u,I'|S}(x_u,x_{I'},x_S)$ can appear at most $2^{|I|-|I'|}$ times in the sum on the last line. Therefore,
\[\nu_{u,I'|S} \geq \sum_{\substack{x_u,x_I,x_S\\q(x_u,x_I|x_S)=I'}} \frac{\mathbb{P}(X_S=x_S)}{2^{|I|+1}} \nu_{u,I'|S}(x_u,x_{I'},x_S) \geq \frac{1}{(4|I|)^{2^s}} \left( \frac{1}{|I|} \right)^{2^s(2^s+1)} \nu_{u,I|S}.\]
\end{proof}

\section{Proof of Theorem \ref{thm:algorithm}}

Most of the results in this section are restatements of results in \cite{hamilton2017information}, with small modifications. Hence, most of the proofs in this section reuse the language of the proofs in \cite{hamilton2017information} verbatim.

Let $A$ be the event that for all $u$, $I$, and $S$ with $|I| \leq 2^s$ and $|S| \leq L$ simultaneously, $|\nu_{u,I|S}-\hat{\nu}_{u,I|S}| < \tau'/2$. Then Lemma \ref{lemma_5_3} gives a result on the number of samples required for event $A$ to hold.

\begin{lemma}[Corollary of Lemma 5.3 in \cite{hamilton2017information}]
\label{lemma_5_3}
If the number of samples is larger than
\[\frac{60 \cdot 2^{2L}}{(\tau')^2 (e^{-2\gamma})^{2L}} \left( \log(1/\omega) + \log(L+2^s+1) + (L+2^s+1)\log(2n) + \log 2 \right),\]
then $\mathbb{P}(A) \geq 1-\omega$.
\end{lemma}

Now, Lemmas \ref{lemma_5_4}-\ref{lemma_5_6} provide the ingredients necessary to prove correctness, assuming that event $A$ holds.

\begin{lemma}[Analogue of Lemma 5.4 in \cite{hamilton2017information}]
\label{lemma_5_4}
Assume that the event $A$ holds. Then every time variables are added to $S$ in Step 2 of the algorithm, the mutual information $I(X_u;X_S)$ increases by at least $(\tau')^2/8$.
\end{lemma}
\begin{proof}
Following the proof of Lemma 5.4 in \cite{hamilton2017information}, we have that when event $A$ holds,
\[\sqrt{\frac{1}{2} \cdot I(X_u;X_I|X_S)} \geq \frac{1}{2} \nu_{u,I|S} \geq \frac{1}{2} (\hat{\nu}_{u,I|S}-\tau'/2).\]
The algorithm only adds variables to $S$ if $\hat{\nu}_{u,I|S} > \tau'$, so
\[I(X_u;X_I|X_S) \geq \frac{1}{2}(\hat{\nu}_{u,I|S} - \tau'/2)^2 \geq  \frac{1}{2} (\tau' - \tau'/2)^2 = (\tau')^2/8.\]
\end{proof}

\begin{lemma}[Analogue of Lemma 5.5 in \cite{hamilton2017information}]
\label{lemma_5_5}
Assume that the event $A$ holds. Then at the end of Step 2 $S$ contains all of the neighbors of $u$.
\end{lemma} 
\begin{proof}
Following the proof of Lemma 5.5 in \cite{hamilton2017information}, we have that Step 2 ended either because $|S| > L$ or because there was no set of variables $I \subseteq [n] \setminus (\{u\} \cup S)$ with $\hat{\nu}_{u,I|S} > \tau'$.

If $|S| > L$, we have by Lemma \ref{lemma_5_4} that $I(X_u; X_S) > L \cdot (\tau')^2/8 = 1$. However, because $X_u$ is a binary variable, we also have $1 \geq H(X_u) \geq I(X_u;X_S)$, so we obtain a contradiction.

Suppose then that $|S| \leq L$, but that there was no set of variables $I \subset [n] \setminus (\{u\} \cup S)$ with $|I| \leq 2^s$ and $\hat{\nu}_{u,I|S} > \tau'$. If $S$ does not contain all of the neighbors of $u$, then we know by Theorem \ref{thm:main_nu_bound} that there exists a set of variables $I \subseteq [n] \setminus (\{u\} \cup S)$ with $|I| \leq 2^s$  with $\nu_{u,I|S} \geq 2\tau'$. Because event $A$ holds, we know that $\hat{\nu}_{u,I|S} \geq \nu_{u,I|S} - \tau'/2 > \tau'$. This contradicts our supposition that there was no such set of variables.

Therefore, $S$ must contain all of the neighbors of $u$.
\end{proof}

\begin{lemma}[Analogue of Lemma 5.6 in \cite{hamilton2017information}]
\label{lemma_5_6}
Assume that the event $A$ holds. If at the start of Step 3 $S$ contains all of the neighbors of $u$, then at the end of Step 3 the remaining set of variables are exactly the neighbors of $u$.
\end{lemma} 
\begin{proof}
Following the proof of Lemma 5.6 in \cite{hamilton2017information}, we have that if event $A$ holds, then
\[\hat{\nu}_{u,i|S\setminus\{i\}} < \nu_{u,i|S\setminus\{i\}} + \tau'/2 \leq \sqrt{\frac{1}{2}I(X_u;X_i|X_S)} + \tau'/2 = \tau'/2\]
for all variables $i$ that are not neighbors of $u$. Then all such variables are pruned. Furthermore, by Theorem \ref{thm:main_nu_bound},
\[\hat{\nu}_{u,i|S\setminus\{i\}} \geq \nu_{u,i|S\setminus\{i\}} - \tau'/2 \geq 2\tau' - \tau'/2 > \tau'\]
for all variables $i$ that are neighbors of $u$, and thus no neighbor is pruned.
\end{proof}

\begin{proof}[Proof of Theorem \ref{thm:algorithm} (Analogue of Theorem 5.7 in \cite{hamilton2017information})]
Event $A$ occurs with probability $1-\omega$ for our choice of $M$. By Lemmas \ref{lemma_5_5} and \ref{lemma_5_6}, the algorithm returns the correct set of neighbors of $u$ for every observed variable $u$. 

To analyze the running time, observe that when running the algorithm at an observed variable $u$, the bottleneck is Step 2, in which there are at most $L$ steps and in which the algorithm must loop over all subsets of vertices in $[n] \setminus \{u\} \setminus S$ of size $2^s$, of which there are $\sum_{l=1}^{2^s} \binom{n}{l} = O(n^{2^s})$ many. Running the algorithm at all observed variables thus takes $O(MLn^{2^s+1})$ time.
\end{proof}

\section{Structure Learning Algorithm of Section \ref{sec:alpha}}

The steps of the structure learning algorithm are:
\begin{enumerate}
  \item Fix parameters $s$, $\tau'(\zeta \cdot \eta)$, $\tau'(\zeta)$, $L$. Fix observed variable $u$. Set $S:=\emptyset$.
  \item While $|S|\leq L$ and there exists a set of observed variables $I \subseteq [n] \setminus \{u\} \setminus S$ of size at most $2^{s}$ such that $\hat{\nu}_{u,I|S} > \tau'(\zeta \cdot \eta)$, set $S:=S\cup I$.
  \item For each $i \in S$, if $\hat{\nu}_{u,I|S\setminus \{i\}} < \tau'(\zeta)$ for all sets of observed variables $I \subseteq [n] \setminus \{u\} \setminus (S \setminus \{i\})$ of size at most $2^s$, mark $i$ for removal from $S$.
  \item Remove from $S$ all variables marked for removal.
  \item Return set $S$ as an estimate of the neighborhood of $u$.
\end{enumerate}
In the algorithm above, we use
\[L = 8/(\tau'(\zeta \cdot \eta))^2, \quad \eta = \frac{1}{\sqrt{3} \cdot 2^{D/2+2^{s}} \cdot D^{2^{s-1}(2^s+2)}},\]
\[\tau'(x) = \frac{1}{(4d)^{2^{s}}} \left(\frac{1}{d}\right)^{2^{s}(2^{s}+1)} \tau(x), \text{ and } \tau(x) = \frac{1}{2} \frac{4x^2 (e^{-2\gamma})^{d+D-1} }{d^{4d} 2^{d+1} \binom{D}{d-1} \gamma e^{2\gamma}}.\]

The main difference in the analysis of this algorithm compared to that of the algorithm in Section \ref{sec:algorithm} is that, at the end of Step 2, $S$ is no longer guaranteed to contain all the neighbors of $u$. Then, a smaller threshold is used in Step 2 compared to Step 3 in order to ensure that $S$ contains enough neighbors of $u$ such that the mutual information proxy with any non-neighbor is small.

\section{Proof of Theorem \ref{thm:recov_alpha}}

See Appendix C for a detailed description of the structure learning algorithm in Section \ref{sec:alpha}.

The correctness of the algorithm is based on the results in Theorem \ref{thm_recov_alpha_nu} and Lemma \ref{lemma_prune_alpha_nu}, which are analogues of Theorem \ref{thm:main_nu_bound} and Lemma \ref{lemma_5_6}. We state these, and then we prove Theorem \ref{thm:recov_alpha} based on them. Then, Section \ref{proof_alpha_1} proves Theorem \ref{thm_recov_alpha_nu} and Section \ref{proof_alpha_2} proves Lemma \ref{lemma_prune_alpha_nu}.

\begin{theorem}[Analogue of Theorem \ref{thm:main_nu_bound}]
\label{thm_recov_alpha_nu}
Fix an observed variable $u$ and a subset of observed variables $S$, such that the two are disjoint. Suppose there exists a neighbor $i$ of $u$ not contained in $S$ such that the MRF of the observed variables contains a Fourier coefficient associated with both $i$ and $u$ that has absolute value at least $\zeta$. Then there exists some subset $I$ of the MRF neighborhood of $u$ with $|I| \leq 2^s$ such that
\[\nu_{u,I|S} \geq \frac{1}{(4d)^{2^{s}}} \left(\frac{1}{d}\right)^{2^{s}(2^{s}+1)} \frac{4\zeta^2 (e^{-2\gamma})^{d+D-1} }{d^{4d} 2^{d+1} \binom{D}{d-1} \gamma e^{2\gamma}} = 2\tau'(\zeta).\]
\end{theorem}

Let $A_{\zeta, \eta}$ be the event that for all $u$, $I$, and $S$ with $|I| \leq 2^s$ and $|S| \leq L$ simultaneously, $|\nu_{u,I|S}-\hat{\nu}_{u,I|S}| < \tau'(\zeta \cdot \eta)/2$.

\begin{lemma}[Analogue of Lemma \ref{lemma_5_6}]
\label{lemma_prune_alpha_nu}
Assume that the event $A_{\zeta, \eta}$ holds. If at the start of Step 3 $S$ contains all of the neighbors of $u$ which are connected to $u$ through a Fourier coefficient of absolute value at least $\zeta \cdot \eta$, then at the end of Step 4 the remaining set of variables is a subset of the neighbors of $u$, such that all neighbors which are connected to $u$ through a Fourier coefficient of absolute value at least $\zeta$ are included in the subset.
\end{lemma}

\begin{proof}[Proof of Theorem \ref{thm:recov_alpha}]
Event $A_{\zeta, \eta}$ occurs with probability $1-\omega$ for our choice of $M$. Then, based on the result of Theorem \ref{thm_recov_alpha_nu}, we have that Lemmas \ref{lemma_5_3}, \ref{lemma_5_4}, and \ref{lemma_5_5} hold exactly as before, with $\tau'(\zeta \cdot \eta)$ instead of $\tau'$, and with the guarantee that at the end of Step 2 $S$ contains all of the neighbors of $u$ wihch are connected to $u$ through a Fourier coefficient of absolute value at least $\zeta \cdot \eta$. Finally, Lemma \ref{lemma_prune_alpha_nu} guarantees that the pruning step results in the desired set of neighbors for every observed variable $u$.

The analysis of the running time is identical to that in Theorem \ref{thm:algorithm}.
\end{proof}

\subsection{Proof of Theorem \ref{thm_recov_alpha_nu}}
\label{proof_alpha_1}

We will argue that Theorem 4.6 in \cite{hamilton2017information} holds in the following modified form, which only requires the existence of one Fourier coefficient that has absolute value at least $\alpha$:
\begin{theorem}[Modification of Theorem 4.6 in \cite{hamilton2017information}]
\label{thm:mod_4_6}
Fix a vertex $u$ and a subset of the vertices $S$ which does not contain the entire neighborhood of $u$, and assume that there exists an $\alpha$-nonvanishing hyperedge containing $u$ and which is not contained in $S \cup \{u\}$. Then taking $I$ uniformly at random from the subsets of the neighbors of $u$ not contained in $S$ of size $s = \min(r-1, |\Gamma(u) \setminus S|)$,
\[\mathbb{E}_I\left[ \sqrt{\frac{1}{2}I(X_u;X_I|X_S)} \right] \geq \mathbb{E}_I[\nu_{u,I|S}] \geq C'(\gamma, K, \alpha)\]
where explicitly
\[C'(\gamma, K, \alpha) := \frac{4\alpha^2\delta^{r+d-1}}{r^{4r}K^{r+1}\binom{D}{r-1}\gamma e^{2\gamma}}.\]
\end{theorem}

Then, this allows us to prove Theorem \ref{thm_recov_alpha_nu} with a proof nearly identical to that of Theorem \ref{thm:main_nu_bound}.

\begin{proof}[Proof of Theorem \ref{thm_recov_alpha_nu}]
Using Theorem \ref{thm:mod_4_6}, we get that there exists some subset $I$ of neighbors of $u$ with $|I| \leq d-1$ such that 
\[\nu_{u,I|S} \geq \frac{4\zeta^2 (e^{-2\gamma})^{d+D-1} }{d^{4d} 2^{d+1} \binom{D}{d-1} \gamma e^{2\gamma}} = 2\tau(\zeta).\]
Then, using Theorem \ref{thm:nu_bound_alone}, we have that there exists some subset $I' \subseteq I$ with $|I'| \leq 2^s$ such that
\[\nu_{u,I'|S} \geq \frac{1}{(4d)^{2^{s}}} \left(\frac{1}{d}\right)^{2^{s}(2^{s}+1)} \frac{4\zeta^2 (e^{-2\gamma})^{d+D-1} }{d^{4d} 2^{d+1} \binom{D}{d-1} \gamma e^{2\gamma}} = 2\tau'(\zeta).\]
\end{proof}

What remains is to show that Theorem \ref{thm:mod_4_6} is true. 
Theorem \ref{thm:mod_4_6} differs from Theorem 4.6 in \cite{hamilton2017information} only in that it requires at least one hyperedge containing $u$ and not contained in $S \cup \{u\}$ to be $\alpha$-nonvanishing, instead of requiring all maximal hyperedges to be $\alpha$-nonvanishing. The proof of Theorem 4.6 in \cite{hamilton2017information} uses the fact that all maximal hyperedges are $\alpha$-nonvanishing in exactly two places: Lemma 3.3 and Lemma 4.5. In both of these lemmas, it can be easily shown that the same result holds even if only one, not necessarily maximal, hyperedge is $\alpha$-nonvanishing. In fact, the original proofs of these lemmas do not make use of the assumption that all maximal hyperedges are $\alpha$-nonvanishing: they only use that there exists a maximal hyperedge that is $\alpha$-nonvanishing. 

We now reprove Lemma 3.3 and Lemma 4.5 in \cite{hamilton2017information} under the new assumption. These proofs contain only small modifications compared to the original proofs. Hence, most of the content of these proofs is restated, verbatim, from \cite{hamilton2017information}.

Lemma \ref{lemma:clique_noncanc_2} is a trivial modification of Lemma 2.7 in \cite{hamilton2017information}, to allow the tensor which is lower bounded in absolute value by a constant $\kappa$ to be non-maximal. Then, Lemma \ref{re_lemma_3_3} is the equivalent of Lemma 3.3 in \cite{hamilton2017information} and Lemma \ref{re_lemma_4_5} is the equivalent of Lemma 4.5 in \cite{hamilton2017information}, under the assumption that there exists at least one hyperedge containing $u$ that is $\alpha$-nonvanishing. 

\begin{lemma}[Modification of Lemma 2.7 in \cite{hamilton2017information}]
\label{lemma:clique_noncanc_2}
Let $T^{1...s}$ be a centered tensor of dimensions $d_1 \times ... \times d_s$ and suppose there exists at least one entry of $T^{1...s}$ which is lower bounded in absolute value by a constant $\kappa$. For any $1 \leq l \leq r$, and $i_1 < ... < i_l$ such that $\{i_1, ..., i_l\} \neq [s]$, let $T^{i_1....i_l}$ be an arbitrary centered tensor of dimensions $d_{i_1} \times .... \times d_{i_l}$. Let 
\[T(a_1, ..., a_r) = \sum_{l=1}^r \sum_{i_1 < ... < i_l} T^{i_1...i_l}(a_{i_1}, ..., a_{i_l}).\]
Then there exists an entry of $T$ of absolute value lower bounded by $\kappa/s^s$.
\end{lemma}
\begin{proof}
Suppose all entries of $T$ are less than $\kappa/s^s$ in absolute value. Then by Lemma 2.6 in \cite{hamilton2017information}, all the entries of $T^{1...s}$ are less than $\kappa$ in absolute value. This is a contradiction, so there exists an entry of $T$ of absolute value lower bounded by $\kappa/s^s$.
\end{proof}

\begin{lemma}[The statement is the same as that of Lemma 3.3 in \cite{hamilton2017information}]
\label{re_lemma_3_3}
\begin{align*}
&\mathbb{E}_{Y,Z}\left[ \sum_R \sum_{B \neq R} \left( \mathcal{E}_{u,R}^Y - \mathcal{E}_{u,B}^Y - \mathcal{E}_{u,R}^Z + \mathcal{E}_{u,B}^Z \right) \left( \exp(\mathcal{E}_{u,R}^Y + \mathcal{E}_{u,B}^Z) - \exp(\mathcal{E}_{u,B}^Y + \mathcal{E}_{u,R}^Z) \right)\right]\\
&\quad\geq \frac{4\alpha^2 \delta^{r-1}}{r^{2r} e^{2\gamma}}.
\end{align*}
\end{lemma}
\begin{proof}[Proof under relaxed $\alpha$ assumption.]

Following the original proof of Lemma 3.3, set $a = \mathcal{E}_{u,R}^Y + \mathcal{E}_{u,B}^Z$ and $b=\mathcal{E}_{u,B}^Y + \mathcal{E}_{u,R}^Z$, and let $D' = K^3 \exp(2\gamma) \geq D$. Then we have
\begin{align*}
\mathbb{E}_{Y,Z}\left[ \sum_R \sum_{B \neq R} (a-b)(e^a-e^b) \right]
&= \mathbb{E}\left[ \sum_R \sum_{B \neq R} (a-b)\int_b^a e^x dx \right]\\
&\geq \mathbb{E}\left[ \sum_R \sum_{B \neq R} (a-b)^2 e^{-2\gamma} \right] \geq \frac{1}{e^{2\gamma}} \sum_R \sum_{B \neq R} \operatorname{Var}[a-b].
\end{align*}

By Claim 3.4 in \cite{hamilton2017information}, we have 
\[\sum_R \sum_{R \neq B} \operatorname{Var}[a-b] = 4k_u \sum_R \operatorname{Var}[\mathcal{E}_{u,R}^Y].\]

Select a hyperedge $J = \{u, j_1, ..., j_s\}$ containing $u$ with $|J| \leq r$, such that $\theta^{uJ}$ is $\alpha$-nonvanishing. Then we get, for some fixed choice $Y_{\sim J}$,
\[\sum_R \operatorname{Var}[\mathcal{E}_{u,R}^Y] \geq \sum_R \operatorname{Var}[\mathcal{E}_{u,R}^Y | Y_{\sim J}] = \sum_R \operatorname{Var}[T(R, Y_{j_1},...,Y_{j_s})|Y_{\sim J}]\]
where the tensor $T$ is defined by treating $Y_{\sim J}$ as fixed as follows:
\[T(R, Y_{j_1}, ..., Y_{j_s}) = \sum_{l=2}^r \sum_{i_2 < ... < i_l} \theta^{ui_2...i_l}(R,Y_{i_2},...,Y_{i_l}).\]
From Lemma \ref{lemma:clique_noncanc_2}, it follows that $T$ is $\alpha/r^r$-nonvanishing. Then there is a choice of $R$ and $G$ such that $|T(R,G)| \geq \alpha/r^r$. Because $T$ is centered there must be a $G'$ so that $T(R,G')$ has the opposite sign, so $|T(R,G)-T(R,G')| \geq \alpha/r^r$. Then we have
\[\operatorname{Var}[T(R,Y_{j_1},...,Y_{j_s})|Y_{\sim J}] \geq \frac{\alpha^2 \delta^{r-1}}{2r^{2r}}\]
which follows from the fact that $\mathbb{P}(Y_{J \setminus u}=G)$ and $\mathbb{P}(Y_{J \setminus u}=G')$ are both lower bounded by $\delta^{r-1}$, and by then applying Claim 3.5 in \cite{hamilton2017information}. Overall, then, 
\[\mathbb{E}_{Y,Z}\left[ \sum_R \sum_{B \neq R} \left( a-b \right) \left( e^a-e^b \right)\right] \geq \frac{4\alpha^2 \delta^{r-1}}{r^{2r} e^{2\gamma}}.\]

\end{proof}

\begin{lemma}[The statement is the same as that of Lemma 4.5 in \cite{hamilton2017information}]
\label{re_lemma_4_5}
Let $E$ be the event that conditioned on $X_S=x_S$ where $S$ does not contain all the neighbors of $u$, node $u$ is contained in at least one $\alpha/r^r$-nonvanishing hyperedge. Then $\mathbb{P}(E) \geq \delta^d$.
\end{lemma}
\begin{proof}[Proof under relaxed $\alpha$ assumption.]

Following the original proof of Lemma 4.5, when we fix $X_S=x_S$ we obtain a new MRF where the underlying hypergraph is 
\[\mathcal{H}' = ([n] \setminus S, H'), \quad H' = \{h \setminus S | h \in H\}.\]
Let $\phi(h)$ be the image of a hyperedge $h$ in $\mathcal{H}$ in the new hypergraph $\mathcal{H}'$.

Let $h^*$ be a hyperedge in $\mathcal{H}$ that contains $u$ and is $\alpha$-nonvanishing. Let $f_1, ..., f_p$ be the preimages of $\phi(h^*)$ so that without loss of generality $f_1$ is $\alpha$-nonvanishing. Let $J = \cup_{i=1}^p f_i \setminus \{u\}$. Finally let $J_1 = J \cap S = \{i_1, i_2, ..., i_s\}$ and let $J_2 = J \setminus S = \{i_1', i_2', ..., i'_{s'}\}$. We define
\[T(R, a_1, ..., a_s, a'_1, ..., a'_{s'}) = \sum_{i=1}^p \theta^{f_i}\]
which is the clique potential we get on hyperedge $\phi(h^*)$ when we fix each index in $J_1 \subseteq S$ to their corresponding value. 

Because $\theta^{f_1}$ is $\alpha$-nonvanishing, it follows from Lemma \ref{lemma:clique_noncanc_2} that $T$ is $\alpha/r^r$-nonvanishing. Thus there is some setting $a_1^*, ..., a_s^*$ such that the tensor
\[T'(R, a_1', ..., a'_{s'}) = T(R, a_1^*, ..., a_s^*, a'_1, ..., a'_{s'})\]
has at least one entry with absolute value at least $\alpha/r^r$. What remains is to lower bound the probability of this setting. Since $J_1$ is a subset of the neighbors of $u$ we have $|J_1| \leq d$. Thus the probability that $(X_{i_1}, ..., X_{i_s}) = (a_1^*, ..., a_s^*)$ is bounded below by $\delta^s \geq \delta^d$, which completes the proof.
\end{proof}

\subsection{Proof of Lemma \ref{lemma_prune_alpha_nu}}
\label{proof_alpha_2}

The proof of Lemma \ref{lemma_5_6} does not generalize to the setting of Lemma \ref{lemma_prune_alpha_nu} because at the end of Step 2 $S$ is no longer guaranteed to contain the entire neighborhood of $u$.

Instead, the proof of Lemma \ref{lemma_prune_alpha_nu} is based on the following observation: any $\nu_{u,I|S}$, where $I$ is a set of non-neighbors of $u$, is upper bounded within some factor of $\nu_{u,n^*(u) \setminus S|S}$, where $n^*(u)$ is the set of neighbors of $u$. Intuitively, this follows because any information between $u$ and $I$ must pass through the neighbors of $u$. Then, by guaranteeing that $\nu_{u,n^*(u) \setminus S|S}$ is small, we can also guarantee that $\nu_{u,I|S}$ is small. This allows us to guarantee that all non-neighbors of $u$ are pruned.

Lemma \ref{lemma:intermediate_nu} makes formal a version of the upper bound on the mutual information proxy mentioned above. Then, we prove Lemma \ref{lemma_prune_alpha_nu}.

\begin{lemma}
\label{lemma:intermediate_nu}
Let $X \in \mathcal{X}, Y \in \mathcal{Y}, Z \in \mathcal{Z}, S \in \mathcal{S}$ be discrete random variables. Suppose $X$ is conditionally independent of $Z$, given $(Y, S)$. Then 
\[\nu_{X,Z|S} \leq \frac{|\mathcal{Y}|}{|\mathcal{Z}|} \nu_{X,Y|S}.\]
\end{lemma}
\begin{proof}
\begin{align*}
\nu_{X,Z|S}
&= \mathbb{E}_{S} \sum_{x \in \mathcal{X}} \frac{1}{|\mathcal{X}|} \sum_{z \in \mathcal{Z}} \frac{1}{|\mathcal{Z}|} \left| \mathbb{P}(X=x,Z=z|S) - \mathbb{P}(X=x|S)\mathbb{P}(Z=z|S) \right|\\
&= \mathbb{E}_{S} \sum_{x \in \mathcal{X}} \frac{1}{|\mathcal{X}|} \sum_{z \in \mathcal{Z}} \frac{1}{|\mathcal{Z}|} \\
&\qquad \cdot \left| \sum_{y \in \mathcal{Y}} \left(\mathbb{P}(X=x,Y=y, Z=z|S) - \mathbb{P}(X=x|S)\mathbb{P}(Y=y,Z=z|S) \right) \right|\\
&\leq \mathbb{E}_{S} \sum_{x \in \mathcal{X}} \frac{1}{|\mathcal{X}|} \sum_{z \in \mathcal{Z}} \frac{1}{|\mathcal{Z}|}\\
&\qquad \cdot \sum_{y \in \mathcal{Y}} \left|\mathbb{P}(X=x,Y=y, Z=z|S) - \mathbb{P}(X=x|S)\mathbb{P}(Y=y,Z=z|S) \right|\\
&\stackrel{(*)}{=} \mathbb{E}_{S} \sum_{x \in \mathcal{X}} \frac{1}{|\mathcal{X}|} \sum_{z \in \mathcal{Z}} \frac{1}{|\mathcal{Z}|} \\
&\qquad \cdot \sum_{y \in \mathcal{Y}} \mathbb{P}(Z=z|Y=y,S) \left|\mathbb{P}(X=x,Y=y|S) - \mathbb{P}(X=x|S)\mathbb{P}(Y=y|S) \right|\\
&= \frac{|\mathcal{Y}|}{|\mathcal{Z}|} \mathbb{E}_{S}\sum_{x \in \mathcal{X}} \frac{1}{|\mathcal{X}|} \sum_{y \in \mathcal{Y}} \frac{1}{|\mathcal{Y}|} \left|\mathbb{P}(X=x,Y=y|S) - \mathbb{P}(X=x|S)\mathbb{P}(Y=y|S) \right|\\
&= \frac{|\mathcal{Y}|}{|\mathcal{Z}|} \nu_{X,Y|S}
\end{align*}
where in (*) we used that $\mathbb{P}(Z=z|X=x,Y=y,S) = \mathbb{P}(Z=z|Y=y,S)$, because $Z$ is conditionally independent of $X$, given $(Y,S)$.
\end{proof}

\begin{proof}[Proof of Lemma \ref{lemma_prune_alpha_nu}]
Consider any $i \in S$ such that $i$ is not a neighbor of $u$, and let $I$ with $|I| \leq 2^s$ be any subset of $[n] \setminus \{u\} \setminus (S \setminus \{i\})$. Let $I^*$ be the set of neighbors of $u$ not included in $S$. Note that $u$ is conditionally independent of $I$, given $(I^*, S \setminus \{i\})$. Then, by Lemma \ref{lemma:intermediate_nu},
\[\nu_{u,I|S\setminus \{i\}} \leq \frac{2^{|I^*|}}{2^{|I|}} \nu_{u,I^*|S\setminus \{i\}} \leq 2^{D-1} \nu_{u,I^*|S\setminus \{i\}}.\]
By Lemma \ref{lemma_non_cancel_nu_cond_full}, there exists a subset $I^\dagger \subseteq I^*$ with $|I^\dagger| \leq 2^s$ such that
\[\nu_{u,I^\dagger|S \setminus \{i\}} \geq \frac{1}{(4|I^*|)^{2^s}} \left(\frac{1}{|I^*|}\right)^{2^s(2^s+1)}\nu_{u,I^*|S\setminus \{i\}} \geq \frac{1}{(4D)^{2^s}} \left(\frac{1}{D}\right)^{2^s(2^s+1)}\nu_{u,I^*|S\setminus \{i\}}.\]
Then, putting together the two results above,
\[\nu_{u,I|S\setminus \{i\}} \leq 2^{D-1}(4D)^{2^s}D^{2^s(2^s+1)} \nu_{u, I^\dagger|S \setminus \{i\}}.\]
Note that
\[\nu_{u, I^\dagger|S \setminus \{i\}} \leq \hat{\nu}_{u,I^\dagger|S \setminus \{i\}} + \tau'(\zeta \cdot \eta) / 2 \stackrel{(*)}{\leq} \tau'(\zeta \cdot \eta) + \tau'(\zeta \cdot \eta) / 2 = 3 \tau'(\zeta \cdot \eta) / 2\]
where in (*) we used that, if $\hat{\nu}_{u,I^\dagger|S \setminus \{i\}}$ were larger than $\tau'(\zeta \cdot \eta)$, the algorithm would have added $I^\dagger$ to $S$. Then 
\begin{align*}
\nu_{u,I|S\setminus \{i\}}
&\leq 3 \cdot 2^{D-2}(4D)^{2^s}D^{2^s(2^s+1)} \tau'(\zeta \cdot \eta)\\
&=\eta^2 \cdot 3 \cdot 2^{D-2}(4D)^{2^s}D^{2^s(2^s+1)} \tau'(\zeta)\\
&= \tau'(\zeta)/4
\end{align*}
where we used that $\tau'(\zeta \cdot \eta) = \eta^2 \tau'(\zeta)$ and then we replaced $\eta$ by its definition. Putting it all together,
\[\hat{\nu}_{u,I|S\setminus \{i\}} \leq \nu_{u,I|S\setminus \{i\}} + \tau'(\zeta \cdot \eta)/2 \leq \tau'(\zeta)/ 4 + \eta^2 \tau'(\zeta) / 2 < \tau'(\zeta)\]
where we used that $\eta \leq 1$. Therefore, all variables $i \in S$ which are not neighbors of $u$ are pruned.

Consider now variables $i$ which are connected to $u$ through a Fourier coefficient of absolute value at least $\zeta$. We know that all variables connected through a Fourier coefficient at least $\zeta \cdot \eta$ are in $S$, so all variables $i$ must also be in $S$, because $\eta \leq 1$. Then, by Theorem \ref{thm_recov_alpha_nu}, there exists a subset $I$ of $[n] \setminus \{u\} \setminus (S \setminus \{i\})$ with $|I| \leq 2^s$, such that
\[\hat{\nu}_{u,I|S\setminus\{i\}} \geq \nu_{u,I|S\setminus\{i\}} - \tau'(\zeta \cdot \eta)/2 \geq \nu_{u,I|S\setminus\{i\}} - \tau'(\zeta)/2 \stackrel{(\dagger)}{\geq} 2\tau'(\zeta) - \tau'(\zeta)/2 > \tau'(\zeta)\]
where in ($\dagger$) we used the guarantee of Theorem \ref{thm_recov_alpha_nu}, knowing that there exists a variable in $[n] \setminus \{u\} \setminus (S \setminus \{i\})$ connected to $u$ through a Fourier coefficient of absolute value at least $\zeta$: specifically, variable $i$. Therefore, no variables $i \in S$ which are connected to $u$ through a Fourier coefficient of absolute value at least $\zeta$ are pruned.
\end{proof}

\section{Proof of Theorem \ref{thm:alpha_distance}}

Let $\psi$ be the maximum over observed variables of the number of non-zero potentials that include that variable:
\[\psi := \max_{u \in [n]} \sum_{\substack{S \subseteq [n]\\u \in S}} \mathbbm{1}\{\hat{f}(S) \neq 0\}.\]
Theorem \ref{thm:alpha_distance_psi}, stated below, is a stronger version of Theorem \ref{thm:alpha_distance}, in which the upper bound on $\zeta$ depends on $\psi$ instead of $D^d$. This section proves Theorem \ref{thm:alpha_distance_psi}. Note that $\psi \leq \sum_{k=0}^{d-1} \binom{D}{k} \leq D^{d-1}+1 < D^d$, so Theorem \ref{thm:alpha_distance_psi} immediately implies Theorem \ref{thm:alpha_distance}.

\begin{theorem}
\label{thm:alpha_distance_psi}
Fix $\delta > 0$ and $\epsilon > 0$. Suppose that we are given neighborhoods $n(u)$ for every observed variable $u$ satisfying the guarantees of Theorem \ref{thm:recov_alpha}. Suppose that we are given $M$ samples from the RBM, and that we have
\[M = \Omega\left(\gamma^2\ln(8 \cdot n \cdot 2^{D} / \delta)/\epsilon^2\right), \quad \zeta \leq \frac{\sqrt{\epsilon}}{\psi \sqrt{1+e^{2\gamma}}}.\]
Let $z_u$ and $\hat{w}_u$ be the features and the estimate of the weights when the regression algorithm is run at observed variable $u$. Then, with probability at least $1-\delta$, for all variables $u$,
\[\mathbb{E}\left[\left(\mathbb{P}(X_u=1|X_{\setminus u}) - \sigma\left(\hat{w}_u \cdot z_u\right)\right)^2\right] \leq \epsilon.\]
\end{theorem}

Define the empirical risk and the risk, respectively:
\[\hat{\mathcal{L}}(w) = \frac{1}{M} \sum_{i=1}^M l(y^{(i)} (w \cdot z^{(i)})), \quad \mathcal{L}(w) = \mathbb{E}[l(y (w \cdot z))].\]

Following is an outline of the proof of Theorem \ref{thm:alpha_distance_psi}. Lemma \ref{lemma:zeta_upper_bound_kl} bounds the KL divergence between the true predictor and the predictor that uses $\bar{w}$, where $\bar{w} \in \mathbb{R}^{2^{|n(s)|}}$ is the vector of true weights for every subset of $n(u)$, multiplied by two. Unfortunately, the estimate $\hat{w}$ that optimizes the empirical risk will typically not recover the true weights, because $n(u)$ is not the true set of neighbors of $u$. Lemma \ref{lemma:diff_kl} decomposes the KL divergence between the true predictor and the predictor that uses $\hat{w}$ in terms of $\mathcal{L}(\hat{w}) - \mathcal{L}(\bar{w})$ and the KL divergence that we bounded in Lemma \ref{lemma:zeta_upper_bound_kl}. The term $\mathcal{L}(\hat{w}) - \mathcal{L}(\bar{w})$ can be shown to be small through concentration arguments, which are partially given in Lemma \ref{lemma:alpha_w_closeness}. Thus, we obtain a bound on the KL divergence between the true predictor and the predictor that uses $\hat{w}$. Finally, using Lemma \ref{lemma:pinsker}, we bound the mean-squared error of interest in terms of this KL divergence.

We now give the lemmas mentioned above and complete formally the proof of Theorem \ref{thm:alpha_distance}.

\begin{lemma}
\label{lemma:alpha_w_closeness}
With probability at least $1-\rho$ over the samples, we have for all $w \in \mathbb{R}^{2^{|n(u)|}}$ such that $||w||_1 \leq 2\gamma$,
\[\mathcal{L}(w) \leq \hat{\mathcal{L}}(w) + 4\gamma \sqrt{\frac{2\ln(2 \cdot 2^{D})}{M}} + 2\gamma \sqrt{\frac{2\ln(2/\rho)}{M}}.\]
\end{lemma}
\begin{proof}
We have $||z||_\infty \leq 1$, $y \in \{-1, 1\}$, the loss function is $1$-Lipschitz, and our hypothesis set is $w \in \mathbb{R}^{2^{|n(S)|}}$ such that $||w||_1 \leq 2\gamma$. Then the result follows from Lemma 7 in \cite{wu2019sparse}.
\end{proof}

\begin{lemma}[Pinsker's inequality]
\label{lemma:pinsker}
Let $D_{KL}(a,b) = a\ln(a/b)+(1-a)\ln((1-a)/(1-b))$ denote the KL divergence between two Bernoulli distributions $(a,1-a)$, $(b, 1-b)$ with $a,b\in[0,1]$. Then
\[(a-b)^2 \leq \frac{1}{2} D_{KL}(a||b).\]
\end{lemma}

\begin{lemma}[Inverse of Pinsker's inequality; see Lemma 4.1 in \cite{gotze2019higher}]
\label{lemma:kl_upper_bound}
Let $D_{KL}(a,b) = a\ln(a/b)+(1-a)\ln((1-a)/(1-b))$ denote the KL divergence between two Bernoulli distributions $(a,1-a)$, $(b, 1-b)$ with $a,b\in[0,1]$. Then
\[D_{KL}(a,b) \leq \frac{1}{\min(b,1-b)} (a-b)^2.\]
\end{lemma}

\begin{lemma}
\label{lemma:diff_kl}
For any $w \in \mathbb{R}^{2^{|n(u)|}}$ with $||w||_1 \leq 2\gamma$, we have that
\[\mathcal{L}(\hat{w}) - \mathcal{L}(w) = \mathbb{E}_z \left[ D_{KL}\left( \frac{\mathbb{E}[y|z]+1}{2}, \sigma(\hat{w} \cdot z) \right) - D_{KL}\left( \frac{\mathbb{E}[y|z]+1}{2}, \sigma(w \cdot z) \right)\right].\]
\end{lemma}
\begin{proof}
\begin{align*}
\mathcal{L}(\hat{w}) - \mathcal{L}(w)
&= \mathbb{E}_{z,y}\left[ -\frac{y+1}{2} \ln \sigma(\hat{w} \cdot z) - \frac{1-y}{2} \ln (1 - \sigma(\hat{w} \cdot z)) \right]\\
&\quad -\mathbb{E}_{z,y}\left[ -\frac{y+1}{2} \ln \sigma(w \cdot z) - \frac{1-y}{2} \ln (1 - \sigma(w \cdot z)) \right]\\
&= \mathbb{E}_{z}\left[ -\frac{\mathbb{E}[y|z]+1}{2} \ln \sigma(\hat{w} \cdot z) - \frac{1-\mathbb{E}[y|z]}{2} \ln (1 - \sigma(\hat{w} \cdot z)) \right]\\
&\quad -\mathbb{E}_{z}\left[ -\frac{\mathbb{E}[y|z]+1}{2} \ln \sigma(w \cdot z) - \frac{1-\mathbb{E}[y|z]}{2} \ln (1 - \sigma(w \cdot z)) \right]\\
&\quad = \mathbb{E}_{z}\left[ \frac{\mathbb{E}[y|z]+1}{2} \ln \frac{\sigma(w \cdot z)}{\sigma(\hat{w} \cdot z)} + \frac{1-\mathbb{E}[y|z]}{2} \ln \frac{1 - \sigma(w \cdot z)}{1 - \sigma(\hat{w} \cdot z)} \right]\\
&\quad = \mathbb{E}_z \left[ D_{KL}\left( \frac{\mathbb{E}[y|z]+1}{2}, \sigma(\hat{w} \cdot z) \right) - D_{KL}\left( \frac{\mathbb{E}[y|z]+1}{2}, \sigma(w \cdot z) \right)\right].
\end{align*}
\end{proof}

\begin{lemma}
\label{lemma:zeta_upper_bound_kl}
Let $\zeta \leq \frac{\sqrt{\epsilon}}{\psi \sqrt{1+e^{2\gamma}}}$ and $\bar{w} \in \mathbb{R}^{2^{|n(u)|}}$ with $\bar{w}_S = 2 \hat{f}(S)$ for all $S \subseteq n(u)$. Then, for all assignments $z \in \{-1, 1\}^{2^{|n(u)|}}$,
\[D_{KL}\left( \frac{\mathbb{E}[y|z]+1}{2}, \sigma(\bar{w} \cdot z) \right) \leq \epsilon.\]
\end{lemma}
\begin{proof}
By Lemma \ref{lemma:kl_upper_bound}, we have that
\[D_{KL}\left( \frac{\mathbb{E}[y|z]+1}{2}, \sigma(\bar{w} \cdot z) \right) \leq \frac{1}{\min(\sigma(\bar{w} \cdot z), 1-\sigma(\bar{w} \cdot z))} \left( \frac{\mathbb{E}[y|z]+1}{2} - \sigma(\bar{w} \cdot z) \right)^2.\]

Note that $\mathbb{E}[y|z^*] = 2 \sigma(w^* \cdot z^*) - 1$ for $w^*$ and $z^*$ corresponding to the true neighborhood of $u$, and that $||w^*||_1 \leq 2\gamma$. Note that $\bar{w}_S=w^*_S$ for all $S \subseteq n(u)$. Also note that $\min(\sigma(\bar{w} \cdot z), 1-\sigma(\bar{w} \cdot z)) \geq \sigma(-2\gamma) = \frac{1}{1+e^{2\gamma}}$. Then:

\begin{align*}
&D_{KL}\left( \frac{\mathbb{E}[y|z]+1}{2}, \sigma(\bar{w} \cdot z) \right)\\
&\quad \leq (1 + e^{2\gamma})\left( \frac{\mathbb{E}[y|z]+1}{2} - \sigma(\bar{w} \cdot z) \right)^2\\
&\quad\stackrel{(a)}{=} (1 + e^{2\gamma})\left( \mathbb{E}_{z^*|z}\left[\sigma(w^* \cdot z^*) - \sigma(\bar{w} \cdot z) \right]\right)^2\\
&\quad\stackrel{(b)}{\leq} (1 + e^{2\gamma}) \mathbb{E}_{z^* | z} \left( \sigma(w^* \cdot z^*) - \sigma(\bar{w} \cdot z) \right)^2\\
&\quad= (1+e^{2\gamma}) \mathbb{E}_{z^* | z} \left( \sigma\left(\sum_{S \subseteq n^*(u)} \hat{f}(S\cup\{u\}) \chi_S(x)\right) - \sigma\left(\sum_{S \subseteq n(u)}  \hat{f}(S\cup\{u\})\chi_S(x)\right) \right)^2\\
&\quad\stackrel{(c)}{\leq} (1+e^{2\gamma}) \mathbb{E}_{z^* | z} \left( \sum_{S \subseteq n^*(u)} \hat{f}(S\cup\{u\}) \chi_S(x) - \sum_{S \subseteq n(u)}  \hat{f}(S\cup\{u\}) \chi_S(x) \right)^2\\
&\quad= (1+e^{2\gamma}) \mathbb{E}_{z^* | z} \left( \sum_{\substack{S \subseteq n^*(u)\\S \not\subseteq n(u)}} \hat{f}(S\cup\{u\}) \chi_S(x) \right)^2\\
&\quad\stackrel{(d)}{\leq} (1+e^{2\gamma}) \psi^2 \zeta^2
\end{align*}
where in (a) we used the law of iterated expectations, in (b) we used Jensen's inequality, in (c) we used that $\sigma$ is $1$-Lipschitz, and in (d) we used that the Fourier coefficients that we sum over are all upper bounded in absolute value by $\zeta$ (otherwise the corresponding sets $S$ would need to be included in $n(u)$, by the assumption that $n(u)$ contains all the neighbors connected to $u$ through a Fourier coefficient of absolute value at least $\zeta$). Therefore, setting $\zeta \leq \frac{\sqrt{\epsilon}}{\psi \sqrt{1+e^{2\gamma}}}$ achieves error $\epsilon$.
\end{proof}

\begin{proof}[Proof of Theorem \ref{thm:alpha_distance_psi}]
Let $M \geq C \cdot \gamma^2\ln(8 \cdot n \cdot 2^{D} / \delta)/\epsilon^2$, for some global constant $C$. Then, by Lemma \ref{lemma:alpha_w_closeness}, with probability at least $1-\delta/(2n)$, for all $w \in \mathbb{R}^{2^{|n(u)|}}$ such that $||w||_1 \leq 2\gamma$,
\[\mathcal{L}(\hat{w}) \leq \hat{\mathcal{L}}(\hat{w}) + \epsilon / 2.\]
Note that $l(y(w\cdot z)) = \ln(1+e^{y(w \cdot z)})$ is bounded because $|y (w \cdot z)| \leq 2\gamma$, and $|\ln(1+e^{-2\gamma})-\ln(1+e^{2\gamma})| \leq 4\gamma$ because the function is $1$-Lipschitz. Then, by Hoeffding's inequality, $\mathbb{P}(\hat{\mathcal{L}}(w)-\mathcal{L}(w) \geq t) \leq e^{-2Mt^2/(4\gamma)^2}$. Then, for $M \geq C' \cdot \gamma^2 \ln(2n/\delta)/\epsilon^2$ for some global constant $C'$, with probability at least $1-\delta/(2n)$,
\[\hat{\mathcal{L}}(w) \leq \mathcal{L}(w) + \epsilon / 2.\]
Then the following holds with probability at least $1-\delta/n$ for any $w \in \mathbb{R}^{2^{|n(u)|}}$ with $||w||_1 \leq 2\gamma$:
\[\mathcal{L}(\hat{w}) \leq \hat{\mathcal{L}}(\hat{w}) + \epsilon / 2 \leq \hat{\mathcal{L}}(w) + \epsilon/2 \leq \mathcal{L}(w) + \epsilon.\]
Then we have
\begin{align*}
&\mathbb{E}\left[\left(\mathbb{P}(X_u=1|X_{[n] \setminus \{u\}}) - \sigma\left(\hat{w} \cdot z\right)\right)^2\right]\\
&\quad\stackrel{(a)}{\leq} \frac{1}{2} \mathbb{E}\left[ D_{KL}\left( \mathbb{P}(X_u=1|X_{[n] \setminus \{u\}}), \sigma\left(\hat{w} \cdot z\right) \right) \right]\\
&\quad\stackrel{(b)}{=} \frac{1}{2} (\mathcal{L}(\hat{w}) - \mathcal{L}(\bar{w})) + \frac{1}{2} \mathbb{E}\left[ D_{KL}\left( \mathbb{P}(X_u=1|X_{[n] \setminus \{u\}}), \sigma\left(\bar{w} \cdot z\right) \right) \right]\\
&\quad\stackrel{(c)}{\leq} \frac{1}{2} (\mathcal{L}(\hat{w}) - \mathcal{L}(\bar{w})) + \frac{1}{2} \epsilon\\
&\quad\stackrel{(d)}{\leq} \epsilon
\end{align*}
where in (a) we used Lemma \ref{lemma:pinsker}, in (b) we used Lemma \ref{lemma:diff_kl}, in (c) we used Lemma \ref{lemma:zeta_upper_bound_kl}, and in (d) we used that $\mathcal{L}(\hat{w}) - \mathcal{L}(\bar{w}) \leq \epsilon$.

By a union bound, this holds for all variables $u$ with probability at least $1-\delta$.
\end{proof}

\section{A weaker width suffices}
\label{sec:app_width}

The sample complexity of the algorithm in Section \ref{sec:algorithm} depends on $\gamma$, the width of the MRF of the observed variables. A priori, it is unclear how large $\gamma$ can be. Ideally, we would have an upper bound on $\gamma$ in terms of parameters that are natural to the RBM, such as the width of the RBM $\beta^*$.

In Section \ref{example:beta_exponential}, we give an example of an RBM for which $\beta^* = d \ln d$ and $\gamma$ is linear in $\beta^*$ and exponential in $d$. This example shows that a bound $\gamma \leq \beta^*$ between the widths of the MRF and of the RBM does not generally hold.

However, we show that a bound $\gamma^* \leq \beta^*$ holds for a modified width $\gamma^*$ of the MRF. Then, we show that $\gamma$ can be replaced with $\gamma^*$ everywhere in the analysis of our algorithm (and of that of \cite{hamilton2017information}), without any other change in its guarantees.

$\gamma^*$ is always less than or equal to $\gamma$, and, as we discussed, sometimes strictly less than $\gamma$. Hence, the former dependency on $\gamma$ was suboptimal. By replacing $\gamma$ with $\gamma^*$, we improve the sample complexity, and we make the dependency interpretable in terms of the width of the RBM.

\subsection{Main result}

Let $\gamma^*$ be the modified width of an MRF, defined as
\[\gamma^* := \max_{u \in [n]} \max_{I \subseteq [n] \setminus \{u\}} \max_{x \in \{-1,1\}^n} \left| \sum_{S \subseteq I} \hat{f}(S \cup \{u\}) \chi_{S \cup \{u\}}(x) \right|.\]
Whereas $\gamma$ is a sum of absolute values of Fourier coefficients, $\gamma^*$ requires the signs of the Fourier coefficients that it sums over to be consistent with some assignment $x \in \{-1,1\}^n$. Note that it is always the case that $\gamma^* \leq \gamma$.

Lemma \ref{lemma:bound_gamma} shows that $\gamma^* \leq \beta^*$. Then, in Section \ref{sec:interchangeabe} we argue that $\gamma$ and $\gamma^*$ are interchangeable for the guarantees of the algorithm in \cite{hamilton2017information}, and implicitly for the guarantees of the algorithm in Section \ref{sec:algorithm}. Finally, in Section \ref{example:beta_exponential} we give an example of an RBM for which $\gamma$ is linear in $\beta^*$ and exponential in $d$.

\begin{lemma}
\label{lemma:bound_gamma}
Consider an RBM with width $\beta^*$, and let $\gamma^*$ be the modified width of the MRF of the observed variables. Then $\gamma^* \leq \beta^*$.
\end{lemma}
\begin{proof}
We have
\begin{align*}
\mathbb{P}(X_u=x_u|X_{[n] \setminus \{u\}} = x_{[n] \setminus \{u\}})
&= \frac{\exp\left( \sum_{\substack{S \subseteq [n]\\u \in S}} \hat{f}(S) \chi_{S}(x) \right)}{\exp\left(- \sum_{\substack{S \subseteq [n]\\u \in S}} \hat{f}(S) \chi_{S}(x) \right) + \exp\left( \sum_{\substack{S \subseteq [n]\\u \in S}} \hat{f}(S) \chi_{S}(x) \right)}\\
&= \sigma\left( 2\sum_{\substack{S \subseteq [n]\\u \in S}} \hat{f}(S) \chi_{S}(x) \right).
\end{align*}
On the other hand, we have
\[\sigma(-2\beta^*) \leq \mathbb{P}(X_u=x_u|X_{[n] \setminus \{u\}} = x_{[n] \setminus \{u\}}) \leq \sigma(2\beta^*).\]
Therefore, by the monotonicity of the sigmoid function, we have for all $x \in \{-1,1\}^{n}$,
\[-\beta^* \leq \sum_{\substack{S \subseteq [n]\\u \in S}} \hat{f}(S) \chi_{S}(x) \leq \beta^*,\]
or equivalently,
\[-\beta^* \leq \sum_{S \subseteq [n] \setminus \{u\}} \hat{f}(S \cup \{u\}) \chi_{S \cup \{u\}}(x) \leq \beta^*.\]
Denote $\phi(x_1, ..., x_n) = \sum_{S \subseteq [n] \setminus \{u\}} \hat{f}(S \cup \{u\}) \chi_{S \cup \{u\}}(x)$. Then the following marginalization result holds for any $i \neq u$:
\begin{align*}
&\sum_{\substack{S \subseteq [n] \setminus \{u, i\}}} \hat{f}(S\cup\{u\}) \chi_{S\cup\{u\}}(x)\\
&\quad= \frac{\phi(x_1, ..., x_{i-1}, -1, x_{i+1}, ..., x_n) + \phi(x_1, ..., x_{i-1}, 1, x_{i+1}, ..., x_n)}{2}.
\end{align*}
Because the lower bound $-\beta$ and upper bound $\beta$ apply to each $\phi(x_1, ..., x_n)$, we get that the same bounds apply to the marginalized value:
\[-\beta^* \leq \sum_{\substack{S \subseteq [n] \setminus \{u, i\}}} \hat{f}(S\cup\{i\}) \chi_{S\cup\{i\}}(x) \leq \beta^*.\]
This marginalization result extends trivially to marginalizing multiple variables. Then, by marginalizing all variables $x_i$ for $i \not\in I \cup \{u\}$ for some $I \subseteq [n] \setminus \{u\}$, we get the bounds
\[-\beta^* \leq \sum_{\substack{S \subseteq I}} \hat{f}(S \cup \{u\}) \chi_{S \cup \{u\}}(x) \leq \beta^*.\]
Taking the maximum over $u \in [n]$, $I \in [n] \setminus \{u\}$, and $x \in \{-1,1\}^n$, we get that $\gamma^* \leq \beta^*$.
\end{proof}

\subsection{The same guarantees hold with the weaker width}
\label{sec:interchangeabe}

For the algorithm in Section \ref{sec:algorithm}, the dependence on $\gamma$ comes only from the use of Theorem \ref{thm:main_nu_bound}, for which the dependence on $\gamma$ comes only from the use of Theorem 4.6 in \cite{hamilton2017information}. Hence, it is sufficient to show that Theorem 4.6 in \cite{hamilton2017information} admits the same guarantees when $\gamma$ is replaced with $\gamma^*$.

The modifications that need to be made to the proof of Theorem 4.6 in \cite{hamilton2017information} are trivial: it is sufficient to replace every occurence of the symbol $\gamma$ with the symbol $\gamma^*$. This is because the proof does not use any property of $\gamma$ that is not also a property of $\gamma^*$.

In the rest of this section, we briefly review the occurences of $\gamma$ in the proof of Theorem 4.6 in \cite{hamilton2017information} and argue that they can be replaced with $\gamma^*$. Toward this goal, the rest of this section will use the notation of \cite{hamilton2017information}. We direct the reader to that paper for more information.

We first define $\gamma^*$ in the setting of \cite{hamilton2017information}. We have 
\[\gamma^* := \max_{u \in [n]} \max_{I \in [n] \setminus \{u\}} \max_{X_1 \in [k_1], ..., X_n \in [k_n]} \left| \sum_{l=1}^r \sum_{i_2 < ... < i_l} \mathbbm{1}_{\{i_2...i_l\} \subseteq I}\theta^{ui_2...i_l}(X_u,X_{i_2},...,X_{i_l}) \right|\]
and
\[\delta^* := \frac{1}{K} \exp(-2\gamma^*).\]
With these definitions, for any variable $X_u$ and assignment $R$, we have for its neighborhood $X_U$ that
\[\mathbb{P}(X_u=R|X_U) \geq \frac{\exp(-\gamma^*)}{K \exp(\gamma^*)} = \frac{1}{K} \exp(-2\gamma^*) = \delta^*.\]
Similarly to \cite{hamilton2017information}, we also have that that if we pick any variable $X_i$ and consider the new MRF given by conditioning on a fixed assignment of $X_i$, then the value of $\gamma^*$ for the new MRF is non-increasing.

$\gamma$ and $\delta$ appear in the proof of Theorem 4.6 in \cite{hamilton2017information} as part of Lemma 3.1, Lemma 3.3, Lemma 4.1, and Lemma 4.5. We now aruge, for each of these lemmas, that $\gamma$ and $\delta$ can be replaced with $\gamma^*$ and $\delta^*$, respectively.

\textbf{Lemma 3.1 in \cite{hamilton2017information}}. $\gamma$ is used as part of the upper bound $|\Phi(R,I,X_i)| \leq \gamma \binom{D}{r-1}$, which is used to conclude that the total amount wagered is at most $\gamma K \binom{D}{r-1}$. The upper bound follows from the derivation
\begin{align*}
|\Phi(R, I, X_i)|
&= \left| \sum_{l=1}^s C_{u,l,s} \sum_{i_1 < i_2 < ... < i_l} \mathbbm{1}_{\{i_1...i_l\} \subseteq I} \theta^{ui_1...i_l}(R, X_{i_1}, ..., X_{i_l}) \right|\\
&\leq \binom{D}{r-1} \left| \sum_{l=1}^s \sum_{i_1 < i_2 < ... < i_l} \mathbbm{1}_{\{i_1...i_l\} \subseteq I} \theta^{ui_1...i_l}(R, X_{i_1}, ..., X_{i_l}) \right|\\
& \leq \gamma \binom{D}{r-1}.
\end{align*}
By the definition of $\gamma^*$, the second inequality holds exactly the same with $\gamma^*$, so we also get that $|\Phi(R,I,X_i)| \leq \gamma^* \binom{D}{r-1}$. Then, the total amount wagered is at most $\gamma^* K \binom{D}{r-1}$.

\textbf{Lemma 3.3 in \cite{hamilton2017information}}. This lemma gives a lower bound of $\frac{4\alpha^2 \delta^{r-1}}{r^{2r}e^{2\gamma}}$ on an expectation of interest. We want to replace $\delta$ with $\delta^*$ in the numerator and $\gamma$ with $\gamma^*$ in the denominator.

For the numerator, $\delta$ comes from the lower bounds $\mathbb{P}(Y_{J \setminus u}= G) \geq \delta^{r-1}$ and $\mathbb{P}(Y_{J \setminus u}=G') \geq \delta^{r-1}$. Note that $Y$ is identical in distribution to $X$, the vector of random variables of the MRF. Let $S \subseteq [n]$, $i\in S$, and let $n^*(i)$ denote the set of neighbors of variable $X_i$. Then the lower bounds mentioned above come from the following marginalization argument:
\begin{align*}
\mathbb{P}(X_S = x_S) 
&= \mathbb{P}(X_{i}=x_{i}|X_{S\setminus i} = x_{S\setminus i}) \mathbb{P}(X_{S\setminus i} = x_{S\setminus i})\\
&= \Bigg(\sum_{x_{n^*(i) \setminus S}} \mathbb{P}(X_{i}=x_{i}|X_{n^*(i) \cap S}=x_{n^*(i) \cap S}, X_{n^*(i) \setminus S} = x_{n^*(i) \setminus S}) \\
&\qquad \cdot \mathbb{P}(X_{n^*(i) \setminus S}=x_{n^*(i) \setminus S} | X_{S\setminus i} = x_{S\setminus i}) \Bigg) \mathbb{P}(X_{S\setminus i} = x_{S\setminus i})\\
&\geq \Bigg(\sum_{x_{n^*(i) \setminus S}} \delta \cdot \mathbb{P}(X_{n^*(i) \setminus S}=x_{n^*(i) \setminus S} | X_{S\setminus i} = x_{S\setminus i}) \Bigg) \mathbb{P}(X_{S\setminus i} = x_{S\setminus i})\\
&= \delta \cdot \mathbb{P}(X_{S\setminus i} = x_{S\setminus i}).
\end{align*}
By applying the bound recursively, we obtain $\mathbb{P}(X_S = x_S) \geq \delta^{|S|}$. Then, because $|J \setminus u| \leq r-1$, we get the desired lower bound of $\delta^{r-1}$. Note, however, that the inequality step in the derivation above also holds for $\delta^*$, as it only uses that $\mathbb{P}(X_u=R|X_U) \geq \delta^*$. Therefore, we can use $(\delta^*)^{r-1}$ in the numerator.

For the denominator, $e^{2\gamma}$ comes from the lower bounds $\mathcal{E}_{u,R}^Y + \mathcal{E}_{u,B}^Z \geq -2\gamma$ and $\mathcal{E}_{u,B}^Y + \mathcal{E}_{u,R}^Z \geq -2\gamma$. Recall that 
\[\mathcal{E}_{u,R}^X = \sum_{l=1}^r \sum_{i_2 < ... < i_l} \theta^{ui_2...i_l}(R, X_{i_2}, ..., X_{i_l}).\]
Then, by the definition of $\gamma^*$, these lower bounds also hold trivially with $\gamma^*$, so we can use $e^{2\gamma^*}$ in the denominator. 

\textbf{Lemma 4.1 in \cite{hamilton2017information}}. In this lemma, $\gamma$ appears in an upper bound of $\gamma K \binom{D}{r-1}$ on the total amount wagered. We showed in Lemma 3.1 that the total amount wagered is at most $\gamma^* K \binom{D}{r-1}$, so we can use $\gamma^*$ instead of $\gamma$.

\textbf{Lemma 4.5 in \cite{hamilton2017information}}. In this lemma, $\delta$ appears in the lower bound $\mathbb{P}(X_{i_1}=a_1^*, ..., X_{i_s}=a_s^*) \geq \delta^s$, which also holds with $\delta^*$ instead of $\delta$ by the argument that $\mathbb{P}(X_S=x_s) \geq (\delta^*)^{|S|}$ that we developed in our description of Lemma 3.3.

Therefore, it is possible to reaplce $\gamma$ with $\gamma^*$ and $\delta$ with $\delta^*$ everywhere in the proof of Theorem 4.6 in \cite{hamilton2017information}, and implicitly also in all the proofs of the the algorithm in Section \ref{sec:algorithm}.

\subsection{Example of RBM with large width}
\label{example:beta_exponential}

This section gives an example of an RBM with width linear in $\beta^*$ and exponential in $d$. The RBM consists of a single latent variable connected to $d$ observed variables. There are no external fields, and all the interactions have the same value. Note that, in this case, each interaction is equal to $\frac{\beta^*}{d}$, where $\beta^*$ is the width of the RBM.

For this RBM, the MRF induced by the observed variables has a probability mass function
\[\mathbb{P}(X=x) \propto \exp\left( \rho\left( \frac{\beta^*}{d} (x_1 + ... + x_d) \right) \right).\]
The analysis of $\gamma$ for this MRF is based on the fact that, for large arguments, the function $\rho$ is well approximated by the absolute value function, for which the Fourier coefficients can be explicitly calculated.

Lemma \ref{lemma_abs_val_coeff} gives a lower bound on the ``width'' corresponding to the Fourier coefficients of the absolute value function applied to $x_1 + ... + x_d$. Then, Lemma \ref{lemma_exp_gamma_counterexample} gives a lower bound on $\gamma$ for the RBM described above, in the case when $\beta^* \geq d \ln d$. This lower bound is linear in $\beta^*$ and exponential in $d$.

\begin{lemma}
\label{lemma_abs_val_coeff}
Let $g: \{-1,1\}^d \to \mathbb{R}$ with $g(x) = |x_1 + ... + x_d|$. Let $\hat{g}$ be the Fourrier coefficients of $g$. Then, for $d$ multiple of $4$ plus $1$, for all $u \in [d]$,
\[\sum_{\substack{S \subseteq [d] \\ u \in S}} |\hat{g}(S)| \geq \frac{2^{(d-1)/2}}{2\sqrt{d-1}}.\]
\end{lemma}
\begin{proof}
Note that, for $x \in \{-1, 1\}^d$, we have
\[|x_1 + ... + x_d| = \operatorname{Maj}_d(x_1, ..., x_d) \cdot (x_1 + ... + x_d)\]
where $\operatorname{Maj}_d(x_1, ..., x_d)$ is the majority function, equal to $1$ if more than half of the arguments are $1$ and equal to $-1$ otherwise. Because $d$ is odd, the definition is non-ambiguous. The Fourier coefficients of $\operatorname{Maj}_d(x_1, ..., x_d)$ are known to be (see Chapter 5.3 in \cite{o2014analysis}):
\[\hat{\operatorname{Maj}}_d(S) = \begin{cases}
(-1)^{(|S|-1)/2} \frac{1}{2^{d-1}} \binom{d-1}{(d-1)/2} \frac{\binom{(d-1)/2}{(|S|-1)/2}}{\binom{d-1}{|S|-1}} & \text{if } |S| \text{ odd}\\
0 & \text{if } |S| \text{ even}
\end{cases}\]
Let $h_i(x_1, ..., x_d) = \operatorname{Maj}_d(x_1, ..., x_d) \cdot x_i$. The Fourier coefficients $\hat{h}_i$ are obtained from the Fourier coefficients $\hat{\operatorname{Maj}}_d$ by observing the effect of the multiplication by $x_i$: for a set $S$ such that $i \in S$, we get $\hat{h}_i(S) = \hat{\operatorname{Maj}}_d(S \setminus \{i\})$, and for a set $S$ such that $i \not\in S$, we get $\hat{h}_i(S) = \hat{\operatorname{Maj}}_d(S \cup \{i\})$. That is:
\[\hat{h}_i(S) = \begin{cases}
(-1)^{(|S|-2)/2} \frac{1}{2^{d-1}} \binom{d-1}{(d-1)/2} \frac{\binom{(d-1)/2}{(|S|-2)/2}}{\binom{d-1}{|S|-2}} & \text{if } |S| \text{ even and } i \in S\\
(-1)^{(|S|)/2} \frac{1}{2^{d-1}} \binom{d-1}{(d-1)/2} \frac{\binom{(d-1)/2}{|S|/2}}{\binom{d-1}{|S|}} & \text{if } |S| \text{ even and } i \not\in S\\
0 & \text{if } |S| \text{ odd}
\end{cases}\]
Then $\hat{g}$ is simply obtained as $\hat{h}_1 + ... + \hat{h}_d$. This gives:
\[\hat{g}(S) = \begin{cases}
(-1)^{(|S|-2)/2} \frac{1}{2^{d-1}} \binom{d-1}{(d-1)/2} \left(|S| \cdot \frac{\binom{(d-1)/2}{(|S|-2)/2}}{\binom{d-1}{|S|-2}} - (d-|S|) \cdot \frac{\binom{(d-1)/2}{|S|/2}}{\binom{d-1}{|S|}} \right) & \text{if } |S| \text{ even}\\
0 & \text{if } |S| \text{ odd}
\end{cases}\]

We will now develop a lower bound for $\hat{g}(S)$ when $|S|$ is even with $|S| > 0$. Using the fact that $\binom{a}{b} = \binom{a}{b+1} \frac{b+1}{a-b}$, we have that when $|S|$ is even with $|S| > 0$,
\begin{align*}
\frac{\binom{(d-1)/2}{(|S|-2)/2}}{\binom{d-1}{|S|-2}}
&= \frac{\binom{(d-1)/2}{|S|/2}}{\binom{d-1}{|S|}} \cdot \frac{|S|/2}{(d-|S|+1)/2} \cdot \frac{d-|S|+1}{|S|-1} \cdot \frac{d-|S|}{|S|}\\
&= \frac{\binom{(d-1)/2}{|S|/2}}{\binom{d-1}{|S|}} \cdot \frac{d-|S|}{|S|-1}.
\end{align*}
Then, when $|S|$ is even with $|S| > 0$,
\begin{align*}
\hat{g}(S)
&= (-1)^{(|S|-2)/2} \frac{1}{2^{d-1}} \binom{d-1}{(d-1)/2} \frac{\binom{(d-1)/2}{|S|/2}}{\binom{d-1}{|S|}} \left(|S| \frac{d-|S|}{|S|-1}- (d-|S|) \right)\\
&=(-1)^{(|S|-2)/2} \frac{1}{2^{d-1}} \binom{d-1}{(d-1)/2} \frac{\binom{(d-1)/2}{|S|/2}}{\binom{d-1}{|S|}} \frac{d-|S|}{|S|-1}.
\end{align*}

Consider the ratio $\frac{|\hat{g}(S)|}{|\hat{g}(S')|}$ for $|S'|=|S|-2$:
\[\frac{|\hat{g}(S)|}{|\hat{g}(S')|} = \frac{|S|-1}{d-|S|} \frac{\frac{d-|S|}{|S|-1}}{\frac{d-|S|+2}{|S|-3}} = \frac{|S|-3}{d-|S|+2}.\]
This ratio is greater than $1$ for $|S| > (d-1)/2 + 3$ and is less than $1$ for $|S| < (d-1)/2 + 3$. Because we are only interested in $|S|$ even, we see that the largest value of $|S|$ for which the ratio is less than $1$ is $(d-1)/2+2$. Hence, $|\hat{g}(S)|$ is minimized at $|S|=(d-1)/2+2$ when considering $|S|$ even with $|S| > 0$. (The calculation above is not valid for the case $|S|=2$ and $|S'|=0$; however, it is easy to verify explicitly that in that case we have $\frac{|\hat{g}(S)|}{|\hat{g}(S')|}= \frac{1}{d} \leq 1$, so the argument holds.)

It is easy to verify explicitly that at $|S|=(d-1)/2+2$ we have
\[|\hat{g}(S)| = \frac{1}{2^{d-1}} \binom{(d-1)/2}{(d-1)/4}.\]
Then this is a lower bound on all $|\hat{g}(S)|$ where $|S|$ is even with $|S| > 0$. Then, 
\begin{align*}
|\hat{g}(S)|
\geq \frac{1}{2^{d-1}} \binom{(d-1)/2}{(d-1)/4}
\stackrel{(*)}{\geq} \frac{1}{2^{d-1}} \frac{2^{(d-1)/2}}{\sqrt{d-1}}
= \frac{1}{\sqrt{d - 1} \cdot 2^{(d-1)/2}}
\end{align*}
where in (*) we used the central binomial coefficient lower bound $\binom{2n}{n} \geq \frac{4^n}{\sqrt{4n}}$.

Then, for any $u \in [d]$,
\[\sum_{\substack{S \subseteq [d]\\u \in S}} |\hat{g}(S)| \geq 2^{d-2} \cdot \frac{1}{\sqrt{d - 1} \cdot 2^{(d-1)/2}} = \frac{2^{(d-1)/2}}{2\sqrt{d-1}}\]
where we used that the number of subsets $S \subseteq [d]$ with $u \in S$ and with $|S|$ even is $2^{d-2}$.
\end{proof}

\begin{lemma}
\label{lemma_exp_gamma_counterexample}
For any $d \geq 5$ multiple of $4$ plus $1$ and $\beta^* \geq d \ln d$, there exists an RBM of width $\beta^*$ with $d$ observed variables and one latent variable such that, in the MRF of the observed variables,
\[\gamma \geq \beta^* \cdot \frac{2^{(d-1)/2}}{4d^{3/2}}.\]
\end{lemma}
\begin{proof}
Let $f(x) = \rho\left(\frac{\beta^*}{d}(x_1 + ... + x_d)\right)$. Then, for the RBM with one latent variable connected to $d$ observed variables through interactions of value $\frac{\beta}{d}$, we have that 
\[\mathbb{P}(X=x) \propto \exp(f(x)).\]
Note that this RBM has width $\beta^*$.

Let $g(x) = \left| \frac{\beta^*}{d} \left(x_1 + ... + x_d\right) \right|$. Then, if $\hat{f}$ and $\hat{g}$ are the Fourier coefficients corresponding to $f$ and $g$, respectively, we have
\begin{align*}
||\hat{f} - \hat{g}||_2^2
&\stackrel{(a)}{=} \frac{1}{2^d} \sum_{x \in \{-1,1\}^d} (f(x)-g(x))^2\\
&\stackrel{(b)}{\leq} \left(\rho\left(\frac{\beta^*}{d}\right) - \frac{\beta^*}{d}\right)^2\\
&= \left( \log(e^{\beta^*/d} (1 + e^{-2\beta^*/d})) - \frac{\beta^*}{d} \right)^2\\
&= \left( \log(1+e^{-2\beta^*/d}) \right)^2\\
&\stackrel{(c)}{\leq} e^{-4\beta^*/d}
\end{align*}

where in (a) we used Praseval's identity, in (b) we used that $(\rho(y)-|y|)^2$ is largest when $|y|$ is smallest and that $\left|\frac{\beta^*}{d}(x_1+...+x_d)\right| \geq \frac{\beta^*}{d}$ because $d$ is odd, and in (c) we used that $\log(1+x) \leq x$. Then
\[||\hat{f} - \hat{g}||_1 \leq 2^{d/2} ||\hat{f} - \hat{g}||_2 \leq 2^{d/2} e^{-2\beta^*/d}.\]
Note that the Fourier coefficients of $g(x)=\left|\frac{\beta^*}{d}(x_1+...+x_d)\right| = \frac{\beta^*}{d} |x_1 + ... + x_d|$ are $\frac{\beta^*}{d}$ times the Fourier coefficients of $|x_1+...+x_d|$. Then, by applying Lemma \ref{lemma_abs_val_coeff}, we have that
\[\max_{u \in [d]} \sum_{\substack{S \subseteq [d] \\ u \in S}} |\hat{f}(S)| \geq \max_{u \in [d]} \sum_{\substack{S \subseteq [d] \\ u \in S}} |\hat{g}(S)| - 2^{d/2} e^{-2\beta^*/d} \geq \frac{\beta^*}{d} \cdot \frac{2^{(d-1)/2}}{2\sqrt{d}} - 2^{d/2}e^{-2\beta^*/d}.\]
We solve for $\beta^*$ such that the second term is at most half the first term. After some manipulations, we get that
\[2^{d/2}e^{-2\beta^*/d} \leq \frac{1}{2} \frac{\beta^*}{d} \cdot \frac{2^{(d-1)/2}}{2\sqrt{d}} \Longleftrightarrow \beta^* \geq \frac{5}{4}d\ln 2 + \frac{3}{4} d \ln d - \frac{1}{2}d\ln \beta^*.\]
For $d \geq 5$, it suffices to have $\beta^* \geq d \ln d$. Hence, we obtain
\[\max_{u \in [d]} \sum_{\substack{S \subseteq [d] \\ u \in S}} |\hat{f}(S)| \geq \frac{1}{2} \frac{\beta^*}{d} \cdot \frac{2^{(d-1)/2}}{2\sqrt{d}} = \beta^* \cdot \frac{2^{(d-1)/2}}{4d^{3/2}}.\]
\end{proof}

\end{document}